\pdfoutput=1 
\documentclass{article} 
\usepackage{preprint,times}

\usepackage{amsmath,amsfonts,bm}

\def\eqref#1{(\ref{#1})}

\def\1{\bm{1}}

\DeclareMathAlphabet{\mathsfit}{\encodingdefault}{\sfdefault}{m}{sl}
\SetMathAlphabet{\mathsfit}{bold}{\encodingdefault}{\sfdefault}{bx}{n}

\DeclareMathOperator*{\argmax}{arg\,max}
\DeclareMathOperator*{\argmin}{arg\,min}

\DeclareMathOperator{\Tr}{Tr}

\usepackage{hyperref}
\usepackage{url}
\usepackage{amsmath}
\usepackage{amssymb}
\usepackage{mathtools}
\usepackage{amsthm}
\usepackage{bm}

\usepackage{microtype}
\usepackage{graphicx}
\usepackage{subfigure}
\usepackage{booktabs} 
\usepackage{enumitem}
\usepackage{makecell}
\usepackage{algorithm}
\usepackage{algorithmic}

\usepackage{multirow}

\usepackage{enumitem}

\newcommand{\AlgoName}{BOLT}

\theoremstyle{plain}
\newtheorem{theorem}{Theorem}[section]
\newtheorem{proposition}[theorem]{Proposition}
\newtheorem{lemma}[theorem]{Lemma}

\theoremstyle{definition}
\newtheorem{definition}[theorem]{Definition}
\newtheorem{assumption}[theorem]{Assumption}
\theoremstyle{remark}
\newtheorem{remark}[theorem]{Remark}

\title{Time-Varying Bayesian Optimization Without a Metronome}

\author{Anthony Bardou \\
IC, EPFL \\
\texttt{anthony.bardou@epfl.ch} \\
\And
Patrick Thiran \\
IC, EPFL \\
\texttt{patrick.thiran@epfl.ch}
}

\begin{document}

\maketitle

\begin{abstract}
Time-Varying Bayesian Optimization (TVBO) is the go-to framework for optimizing time-varying, expensive, noisy black-box functions. However, most asymptotic guarantees for TVBO algorithms assume observations are acquired at a constant frequency. Since streaming GP inference scales quadratically with dataset size, this assumption is unrealistic in the long run. We relax it and derive the first upper regret bound that explicitly accounts for changes in observation sampling frequency. This analysis yields practical recommendations for dataset sizes and stale-data policies. Experiments on synthetic and real-world problems show that an algorithm (\AlgoName) following these recommendations outperforms state-of-the-art TVBO methods.
\end{abstract}

\section{Introduction} \label{sec:intro}

Bayesian Optimization (BO) uses a surrogate model of the objective function $f$, typically a Gaussian Process (GP), to efficiently learn and optimize it. Its sample efficiency has made BO the standard approach to black-box optimization across diverse applications~\citep{marchant2012bayesian, bardou2025assessing, wang2014bayesian} and motivated extensions to challenging settings such as multi-objective optimization~\citep{daulton2022multi} and high-dimensional input spaces~\citep{bardourelaxing}. Yet only a handful of works study Time-Varying Bayesian Optimization (TVBO), where $f:\mathcal{S}\times\mathcal{T}\to\mathbb{R}$ also depends on time $t\in\mathcal{T}\subseteq\mathbb{R}$, despite the ubiquity of time-varying black-box optimization problems, with occurrences in unmanned aerial vehicles~\citep{melo2021dynamic}, online clustering~\citep{aggarwal2004framework}, and network management~\citep{kim2019dynamic}.

TVBO differs substantially from its static counterpart. In particular, recent work identifies the \textit{response time} $R(n)$ (i.e., the time between consecutive iterations when the dataset has size $n$) as a key algorithmic quantity. Since $R(n)$ typically scales as $\mathcal{O}(n^2)$,\footnote{Inverting the covariance matrix $k(\mathcal{D},\mathcal{D})$ requires $\mathcal{O}(n^3)$ operations, but rank-one updates allow the precision matrix to be updated in $\mathcal{O}(n^2)$ in the streaming setting.} balancing a large dataset against a short response time can yield significant performance gains~\citep{bardou2024too}. However, to the best of our knowledge, no analysis links $R(n)$ to TVBO performance without assuming $R(n)\in\Theta(1)$. This paper fills this gap by:
\begin{itemize}
    \item deriving the first asymptotic regret bound that explicitly relates TVBO performance to the response time $R(n)$ (Section~\ref{sec:main-upper_regret});
    \item translating this analysis into practical recommendations and embedding them in a new algorithm, \AlgoName{} (Sections~\ref{sec:main-upper_regret}, \ref{sec:main-stale}, and~\ref{sec:algo});
    \item providing empirical evidence that \AlgoName{} outperforms state-of-the-art TVBO methods on synthetic and real-world problems (Section~\ref{sec:results}).
\end{itemize}

\section{Background} \label{sec:background}

\subsection{Time-Varying Bayesian Optimization}

Like most BO algorithms, a TVBO algorithm uses a centered GP with covariance function $k$ (denoted by $\mathcal{GP}\left(0, k\right)$) as prior placed on the objective function $f$ with observational noise $\sigma^2_0$. Conditioned on a dataset of $n$ observations $\mathcal{D} = \left\{\left(\bm x_i, t_i, y_i\right)\right\}_{i \in [n]}$, where $y_i = f(\bm x_i, t_i) + \epsilon, \epsilon \sim \mathcal{N}(0, \sigma^2_0)$ for any $i \in [n]$, the posterior satisfies $f(\bm x, t) | \mathcal{D} \sim \mathcal{N}\left(\mu_\mathcal{D}(\bm x, t), \sigma^2_\mathcal{D}(\bm x, t)\right)$ for any $(\bm x, t) \in \mathcal{S} \times \mathcal{T}$, where
\begin{align}
    \mu_\mathcal{D}(\bm x, t) &= \bm k((\bm x, t), \mathcal{D}) \bm \Delta^{-1} \bm y, \label{eq:post_mean}\\
    \sigma^2_\mathcal{D}(\bm x, t) &= k((\bm x, t), (\bm x, t)) - \bm k^\top\left((\bm x, t), \mathcal{D}\right) \bm \Delta^{-1} \bm k\left((\bm x, t), \mathcal{D}\right), \label{eq:post_var}
\end{align}
where $\bm \Delta = \bm K + \sigma^2_0 \bm I_n$, $\bm K = \bm k(\mathcal{D}, \mathcal{D})$, $\bm k(\mathcal{X}, \mathcal{Y}) = \left(k((\bm x_i, t_i), (\bm x_j, t_j))\right)_{(\bm x_i, t_i) \in \mathcal{X}, (\bm x_j, t_j) \in \mathcal{Y}}$, $\bm y = \left(y_i\right)_{y_i \in \mathcal{D}}$ and where $\bm I_n$ is the $n \times n$ identity matrix.

At the $i$-th iteration occurring at time $t_i$, a TVBO algorithm looks for a query $(\bm x_i, t_i)$ that achieves the best exploration-exploitation trade-off. In the BO framework, this problem is addressed by maximizing an acquisition function $\alpha_\mathcal{D} : \mathcal{S} \times \mathcal{T} \to \mathbb{R}$, so that $\bm x_i = \argmax_{\bm x \in \mathcal{S}} \alpha_\mathcal{D}(\bm x, t_i)$. Many acquisition functions have been proposed, the most popular being GPUCB~\citep{gpucb}, Expected Improvement~\citep{ei} and Probability of Improvement~\citep{pi}.

At the $i$-th iteration, the instantaneous performance of a TVBO algorithm is measured by the instantaneous regret $r_i = f(\bm x^*_{i}, t_i) - f(\bm x_{i}, t_i)$, where $\bm x^*_{i} = \argmax_{\bm x \in \mathcal{S}} f(\bm x, t_i)$. The performance up to an horizon $T$ is usually measured by the cumulative regret $R_T = \sum_{i = 1}^T r_i$.

\subsection{The Three Time Scales of a TVBO Problem} \label{sec:background-regimes}

A TVBO problem involves three interacting time scales (which are not given to the algorithm): (i)~a temporal lengthscale $l_T$ that characterizes the rate of variation of $f$, (ii)~the time horizon $H$ of the optimization task, and (iii)~the response time $R(n)$ of the algorithm, which includes the time needed to observe $f$ and to make a new query using a dataset of size $n$. In the following, we denote by $N_H = \left\lfloor H / R(0) \right\rfloor$ the maximal number of observations that could be collected if the computational overhead of GP inference were negligible.

Because $R(n)$ grows with $n$, every TVBO algorithm must implicitly address a trade-off between the information it retains and the frequency at which it can sample $f$. The literature resolves it in three ways: by retaining every observation (growing datasets), by discarding the dataset as a whole at chosen instants (resets), or by discarding selected observations while keeping the rest (removal). Which of these strategies is appropriate depends on the relative magnitudes of $l_T$, $H$ and $R(n)$. We discuss three particular regimes here.

\textbf{Near-Constant Response Time.}
Suppose that $R(N_H)\approx R(0)$. The computational overhead remains negligible even if the algorithm retains every observation that can be collected up to the horizon $H$. This regime typically occurs when objective evaluation is substantially more expensive than GP inference. Growing datasets are then computationally viable, and controlling the dataset size offers little benefit in terms of sampling frequency. Differences between algorithms are therefore driven primarily by the quality of statistical models rather than by their response times.

\textbf{Quasi-Static Problems.}
Suppose that $l_T\gg R(n)$ for a reasonable dataset size $n\ll N_H$. The algorithm can collect many observations before the objective changes substantially. Resets may then remain competitive by treating the problem as a sequence of approximately static optimization tasks, although each reset may discard observations that are still informative. Growing datasets can also model the objective accurately, but their response times may eventually become prohibitive when the horizon is long.

\textbf{General Long-Horizon Problems.}
The most challenging regime occurs when the horizon is long, the response time grows noticeably with the dataset size $n$, and $l_T$ is comparable to $R(n)$ for practically useful values of $n$. Growing datasets may then issue observations at an increasingly low frequency, while resets may repeatedly discard the dataset before enough observations have been collected to build an informative surrogate. Removal is the only one among the three strategies that can, in principle, keep the response time bounded without sacrificing the observations that remain informative, but it raises two questions that the other two strategies answer trivially: how many observations to retain, and which ones to discard.

This paper targets this third regime, and Section~\ref{sec:main} answers both questions in turn.

\subsection{State-of-the-Art of TVBO} \label{sec:background-sota}

\begin{table*}[t]
\caption{Comparison of TVBO solutions. For each method, we report its principal model of temporal variation, how response time enters its theoretical analysis, its stale-data policy, and how its dataset size is controlled. ``Not modeled'' indicates that the method can be executed with a nonconstant response time, but does not explicitly account for it in its analysis.}
\label{tab:sota}
\vskip 0.15in
\begin{center}
\begin{tabular}{ccccc}
\toprule
\makecell{TVBO\\Solution}
&
\makecell{Temporal Model\\or Assumption}
&
\makecell{Response Time\\in Analysis}
&
\makecell{Stale-Data\\Policy}
&
\makecell{Dataset-Size\\Rule}
\\
\midrule
TV-GPUCB
&
\makecell{Markov GP evolution}
&
$R(n)\in\Theta(1)$
&
Growing
&
N/A
\\
R-GPUCB
&
\makecell{Markov GP evolution}
&
$R(n)\in\Theta(1)$
&
Reset
&
Periodic
\\
ABO
&
\makecell{Spatio-temporal GP}
&
Not modeled
&
Growing
&
N/A
\\
SW-GPUCB
&
\makecell{Finite variation budget}
&
$R(n)\in\Theta(1)$
&
Sliding window
&
\makecell{Fixed}
\\
W-GPUCB
&
\makecell{Finite variation budget}
&
$R(n)\in\Theta(1)$
&
\makecell{Growing}
&
N/A
\\
ET-GPUCB
&
\makecell{Markov GP evolution}
&
$R(n)\in\Theta(1)$
&
\makecell{Reset}
&
Event-based
\\
W-DBO
&
\makecell{Spatio-temporal GP}
&
Not modeled
&
\makecell{Removal}
&
\makecell{Adaptive}
\\
\midrule
\AlgoName
&
\makecell{Spatio-temporal GP}
&
Explicit $R(n)$
&
\makecell{Removal}
&
Adaptive
\\
\bottomrule
\end{tabular}
\end{center}
\vskip -0.1in
\end{table*}

Table~\ref{tab:sota} summarizes the main TVBO solutions along the axes that the preceding discussion identifies as decisive: how the response time enters the analysis of each method, its stale-data policy, and how its dataset size is controlled.

Existing regret analyses have primarily considered two models of temporal variation. The first one assumes that the objective follows a persistent Markov evolution. In this setting, the expected cumulative regret is bounded from below linearly, and algorithms with upper bounds of the same asymptotic order, up to polylogarithmic factors, have been proposed~\citep{bogunovic2016time, brunzemaevent}. The second one assumes that the cumulative temporal variation of the objective is sublinear. Sublinear regret can then be achieved when this variation grows sufficiently slowly with the horizon~\citep{zhou2021no, deng2022weighted}. Such condition excludes environments with a nonvanishing temporal variation.

We do not include these regret rates in Table~\ref{tab:sota}, because they hold under fundamentally different models of nonstationarity. Moreover, in a bounded-reward setting, a linear upper bound alone does not distinguish between algorithms. The central distinction relevant to this work is instead whether an algorithm explicitly accounts for the dependence of its response time on its dataset size. Existing algorithms with regret guarantees assume $R(n)\in\Theta(1)$, which amounts to placing themselves in the near-constant regime, whereas methods without this assumption do not model the effect of $R(n)$ on their performance.

Mapping these methods onto the three regimes makes the gap visible. Growing datasets (TV-GPUCB, ABO, W-GPUCB) implicitly assume a near-constant response time, while resets (R-GPUCB, ET-GPUCB) and fixed-size sliding windows (SW-GPUCB) suit quasi-static problems. W-DBO is, to the best of our knowledge, the only method performing relevancy-based removal, but the size of its dataset follows from a user-specified removal budget rather than from an analysis of $R(n)$.

This paper closes that gap. Our analysis considers a GPUCB algorithm equipped with a FIFO policy, which isolates the trade-off between retaining more observations and sampling the evolving objective more frequently. It answers both questions raised above: how many observations to retain (Section 3.2) and which ones to discard (Section 3.3).

\section{Main Results} \label{sec:main}

\subsection{Core Assumptions and Problem Setting} \label{sec:main-assumptions}

In this section, we introduce the core assumptions underpinning this work. Assumption~\ref{ass:gp} is made in all GP-based BO papers as it ensures the existence of the surrogate model. Assumption~\ref{ass:separability} is widely used in the TVBO literature~\citep{bogunovic2016time, nyikosa2018bayesian, bardou2024too}. It captures spatio-temporal dynamics with two correlation functions $k_S$ and $k_T$. Assumption~\ref{ass:kernel_props} requires that $k_S$ and $k_T$ verify some properties. It holds for most kernels used in practice (e.g., Matérn, RBF, Rational Quadratic, sinc) and the condition~\eqref{eq:def_G_assumption} is further discussed in Remark~\ref{rem:G_vs_integrability}. Assumption~\ref{ass:lipschitz} introduces some regularity about the objective function in the spatial domain $\mathcal{S}$. It is used for most regret proofs in the literature~\citep{gpucb, bogunovic2016time}. This is satisfied by sufficiently smooth GPs (e.g., built with Matérn with $\nu > 2$, RBF) but it can fail with rougher GPs (e.g., Ornstein-Uhlenbeck processes built with Matérn kernels where $\nu = 1/2$).

\begin{assumption}[Surrogate Model] \label{ass:gp}
    The objective function $f : \mathcal{S} \times \mathcal{T} \to \mathbb{R}$ is a sample path of $\mathcal{GP}\left(0, k\right)$, where $\mathcal{S} = [0, 1]^d$ and $k : \left(\mathcal{S} \times \mathcal{T}\right)^2 \to \mathbb{R}$ is a covariance function.
\end{assumption}

\begin{assumption}[Covariance Separability] \label{ass:separability}
    The covariance function $k$ has the following form:
    \begin{equation*} \label{eq:tv_cov}
        k((\bm x, t), (\bm x', t')) = \lambda k_S(\bm x, \bm x') k_T(t, t')
    \end{equation*}
    where, without loss of generality, $\lambda = 1$ is the signal variance while $k_S : \mathcal{S} \times \mathcal{S} \to [-1, 1]$ and $k_T : \mathcal{T} \times \mathcal{T} \to [-1, 1]$ are correlation functions in the spatial and temporal domain, respectively.
\end{assumption}

\begin{assumption}[Covariance Properties] \label{ass:kernel_props}
    The correlation function $k_S$ (resp., $k_T$) is real, even, isotropic, stationary and can therefore be rewritten as $k_S(\bm x, \bm x') = k_S(\bm x - \bm x')$ for all $\bm x, \bm x' \in \mathcal{S}$ (resp., $k_T(t, t') = k_T(|t - t'|)$ for all $t, t' \in \mathcal{T}$). Furthermore, there exists a nonnegative, nonincreasing function $g_T$ such that $|k_T(\tau)| \leq g_T(\tau)$ for all $\tau \geq 0$ and such that
    \begin{equation} \label{eq:def_G_assumption}
        G = \sup_{m \geq 1} \frac{1}{R(m)} \int_0^{m R(m)} g_T(\tau) d\tau < \infty,
    \end{equation}
    where $R$ is the response time of the TVBO algorithm (Definition~\ref{def:response_time}). Finally, $\min_{\bm x, \bm x' \in \mathcal{S}} k_S(\bm x, \bm x') > 0$.
\end{assumption}

\begin{assumption}[Lipschitz Continuity in $\mathcal{S}$] \label{ass:lipschitz}
    Let $g \sim \mathcal{GP}(0, k_S)$. Then, for any $\bm x \in \mathcal{S}$, any $L > 0$ and any $i \in [d]$,
    \begin{equation*}
        \mathbb{P}\left[\left|\frac{\partial g(\bm x)}{\partial x_i}\right| > L\right] \leq a e^{-(L/b)^2}.
    \end{equation*}
\end{assumption}

Finally, let us properly define the response time $R(n)$ of a TVBO algorithm, which models the combined costs of observing $f$ and of updating/conditioning the GP on a dataset of size $n$. It satisfies three natural properties: (i)~the objective function cannot be observed at an arbitrarily high frequency because querying it takes a nonzero amount of time, (ii)~the response time is an increasing function of the dataset size $n$ and (iii)~when the dataset size $n$ diverges, the response time diverges as well.

\begin{definition}[Response Time] \label{def:response_time}
    The response time $R : \mathbb{N} \to \mathbb{R}^+$ is a function that returns the time separating two consecutive iterations of a TVBO algorithm with a dataset of size $n \in \mathbb{N}$ such that (i)~$R(0) > 0$, (ii)~$\forall n \in \mathbb{N}, R(n+1) \geq R(n)$ and (iii)~$\lim_{n \to \infty} R(n) = +\infty$.
\end{definition}

Based on Definition~\ref{def:response_time}, we study TVBO algorithms in a realistic, wall-clock–constrained regime where obtaining an informative posterior and issuing the next query are not instantaneous because both function observation and GP inference require non-negligible amounts of time. Concretely, the algorithm queries as soon as it is able to, so that the query times satisfy
\[
t_0=0,\qquad t_m=\sum_{i=1}^m R(n_i)\quad(m\ge1),
\]
where $n_i$ is the dataset size of the TVBO algorithm at iteration $i$.

With this greedy-in-time policy~\citep{nyikosa2018bayesian, bardou2024too}, we consider the \textit{average regret} $R_m / m$ as the metric of performance to minimize after $m$ iterations. This formulation is natural and well-posed: it makes explicit the fundamental trade-off in TVBO between (i)~spending more compute per query to obtain a high-quality posterior (larger \(R(n)\)) and (ii)~issuing more frequent queries to track temporal changes in \(f\).  Any theoretical or empirical comparison of algorithms in this setting must therefore account for the growth and variability of \(R(n)\) and for the intrinsic time-scale of variation of \(f(\cdot,t)\). Ignoring either one of these properties yields conclusions that may be irrelevant in wall-clock constrained, time-varying applications.

\subsection{Regret Analysis Without a Metronome} \label{sec:main-upper_regret}

We now analyze the regret of a TVBO algorithm under the assumptions and definitions introduced in Section~\ref{sec:main-assumptions}. As the cumulative regret of a TVBO algorithm is linear in the worst case when $R(n) \in \Theta(1)$~\citep{bogunovic2016time}, one can expect a similar result when $R(n) \in \omega(1)$.

Understanding how the temporal correlation structure $k_T$ and the response time $R(n)$ jointly affect the regret is essential for designing TVBO algorithms. Existing regret analyses that explicitly manage stale observations have mainly considered policies that reset the entire dataset, either periodically~\citep{bogunovic2016time} or after an event is triggered~\citep{brunzemaevent}, while assuming a constant response time $R(n)\in\Theta(1)$. Such resets are relatively aggressive, as they discard all past observations at once, including recent observations that may remain informative. We consider instead a gentler First-In, First-Out (FIFO) policy: once the dataset reaches a maximal size $n$, each new observation replaces the oldest one, so that the dataset always contains the $n$ most recent observations. This setting allows us to study the trade-off between the information retained from past observations, governed by $k_T$, and the sampling frequency, governed by $R(n)$. This trade-off is formalized by the asymptotic regret bound established in the following theorem.

\begin{theorem} \label{thm:upper_regret_bound}
    Let $\mathcal{A}$ be a TVBO algorithm that uses the GPUCB acquisition function. Let $R(s)$ be the response time of $\mathcal{A}$ for a dataset size $s \in \mathbb{N}$. Let $n$ be the maximal dataset size of $\mathcal{A}$ and let $\|\bm u_n\|_2^2 = \sum_{i = 1}^n k_T^2(iR(n))$. Let $\delta \in (0, 1)$ and let $R_T$ be the cumulative regret of $\mathcal{A}$ after $T$ iterations. Then, with probability $1 - \delta$,
    \begin{equation} \label{eq:regret_bound}
        R_T \leq \sqrt{T C_1 \beta_T\left(\gamma_n + \frac{1}{2 \sigma_0^2} (T-n)(1 - C_2\|\bm u_n\|_2^2)\right)} + \frac{\pi^2}{6},
    \end{equation}
    where $C_1$, $C_2 \in \mathcal{O}(1)$, $\gamma_n \in \mathcal{O}(n)$ and $\beta_T \in \mathcal{O}(\log T)$ are defined in Appendix~\ref{app:dataset_size}.
\end{theorem}

The proof of Theorem~\ref{thm:upper_regret_bound} is provided in Appendix~\ref{app:dataset_size}. The gist of the proof is to first bound the cumulative regret of the TVBO algorithm from above with an expression that involves the mutual information $\gamma_n = I(\bm f_n, \bm y_n)$, where $\bm f_n = \left(f(\bm x_1, t_1), \cdots, f(\bm x_n, t_n)\right)$ and $\bm y_n = \left(y_1, \cdots, y_n\right)$, and the posterior variance of the surrogate model with maximal dataset size $n$. This is achieved with proof techniques introduced by~\citet{gpucb, bogunovic2016time}. Next, an upper bound is derived for the posterior variance that involves $\|\bm u_n\|^2_2$. Combining these two upper bounds yields~\eqref{eq:regret_bound}.

\paragraph{Dependence of the constants and non-negativity.}
The constants $C_1$ and $C_2$ are independent of $n$ and $T$, but they are not universal. In particular, they depend on the observational noise $\sigma^2_0$ and on the spatial and temporal kernels $k_S$ and $k_T$. Furthermore, the quantity under the square root in Theorem~\ref{thm:upper_regret_bound} is shown in Appendix~\ref{app:dataset_size} to be always non-negative~(see Lemma~\ref{lem:bound_on_variance}).

\paragraph{Interpretation.}
For fixed maximal dataset size $n$ and $T\to\infty$, $\gamma_n$ is constant and $\beta_T$ grows polylogarithmically, yielding $R_T\in\widetilde{O}(T)$. This does not establish sublinear regret under persistent temporal variation, consistently with the TVBO literature~\citep{bogunovic2016time, brunzemaevent}. The bound instead reveals how the response time $R(n)$ and temporal kernel $k_T$ jointly affect the choice of maximal dataset size.

In the static case, $k_T(t,t')=1$ and $\gamma_n$ is the mutual information associated with $k_S$. Setting $n=T$ (no dataset cap) in~\eqref{eq:regret_bound} yields $R_T\leq\sqrt{TC_1\beta_T\gamma_T}+\pi^2/6$, recovering the standard sublinear GPUCB bound of Theorem~2 in~\citet{gpucb}.

For fixed $n$ and $T\gg n$, $\gamma_n/T$ vanishes and $T-n\sim T$. The leading term in~\eqref{eq:regret_bound} therefore depends on $n$ through $\sqrt{1-C_2\|\bm u_n\|_2^2}$. Maximizing $\|\bm u_n\|_2^2$ thus minimizes the leading asymptotic term of the bound, motivating the following recommendation.

\paragraph{Recommendation 1.} In the long-horizon regime, a TVBO algorithm should maintain a dataset of size $n^*$, where
\begin{align}
    n^* &= \argmax_{n \in \mathbb{N}} \|\bm u_n\|_2^2 = \argmax_{n \in \mathbb{N}} \sum_{i = 1}^n k_T^2(iR(n)). \label{eq:max_dataset_size}
\end{align}

\paragraph{Computability of Recommendation~1.}
Assumptions~\ref{ass:gp}-\ref{ass:lipschitz} do not guarantee that $\|\bm u_n\|^2_2$ admits a maximizer over $\mathbb{N}$ (see the right panel of Figure~\ref{fig:max_dataset_size} for an empirical example, and Appendix~\ref{app:sufficient_cond} for sufficient conditions). This has no practical consequence: at iteration $T$, the dataset contains at most $T$ observations, so any maximal size larger than $T$ has the same effect as $n^* = T$, and the search in~\eqref{eq:max_dataset_size} can be restricted to $n \in [T]$, over which the maximum is always attained. This range widens as the experiment goes on: it settles on the finite maximizer of $\|\bm u_n\|^2_2$ when one exists, and otherwise keeps recommending a larger dataset, that is, retaining every observation collected so far.

\begin{figure}
    \centering
    \includegraphics[height=4.7cm]{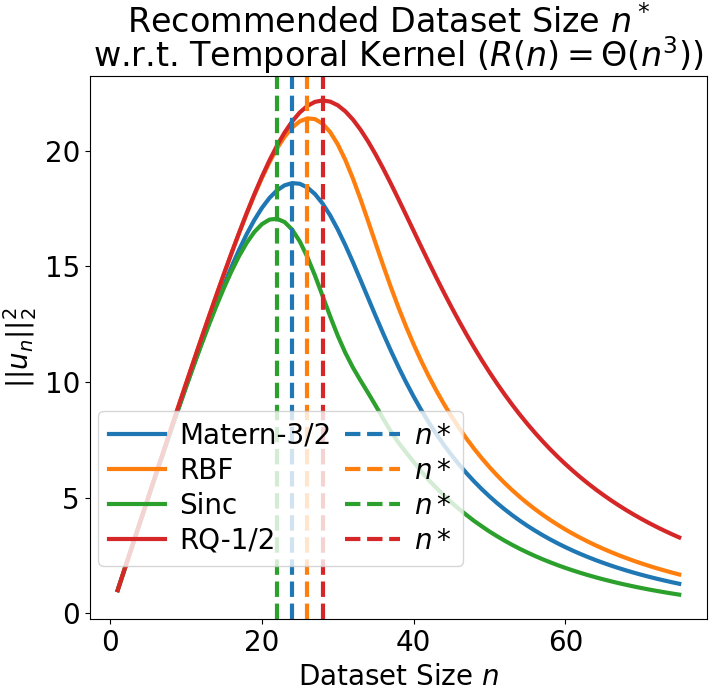}
    \includegraphics[height=4.7cm]{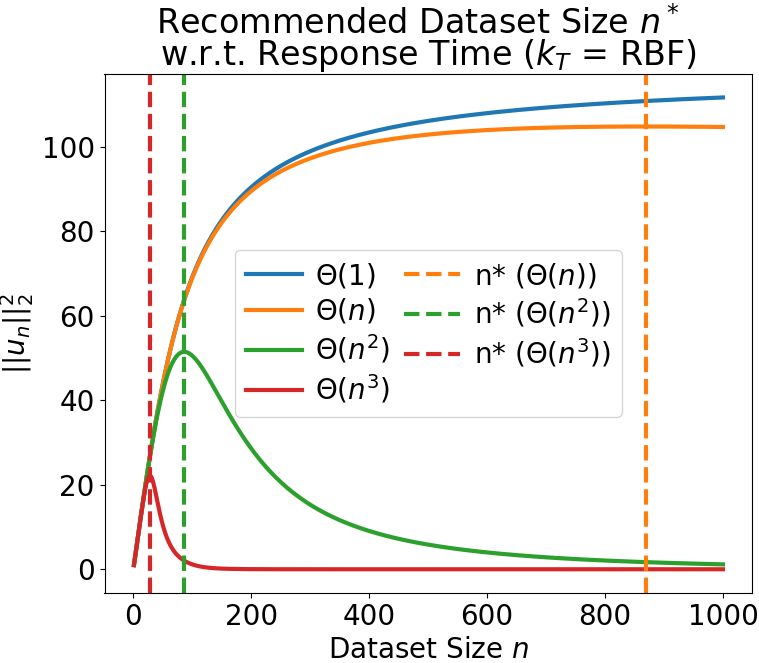}
    \caption{Recommended dataset size $n^* = \argmax_{n \in \mathbb{N}}\|\bm u_n\|^2_2$. (Left)~Recommended dataset size for several common temporal covariance functions $k_T$, under the assumption that the response time is $R(n) \in \Theta(n^3)$. (Right)~The recommended dataset size for several response times $R(n)$, under the assumption that $k_T$ is an RBF covariance function.}
    \label{fig:max_dataset_size}
\end{figure}

Figure~\ref{fig:max_dataset_size} illustrates how $\|\bm u_n\|_2^2$ and $n^*$ behave under different kernels and response times. The left panel shows that the smoothest temporal covariance functions (e.g., RBF or Rational Quadratic) also yield the largest $n^*$, as expected because a smoother GP surrogate can extract useful information from observations collected further in the past. The right panel shows that the response time of the TVBO algorithm has a significant impact on $\|\bm u_n\|_2^2$ and $n^*$: response times that scale slowly with the dataset size $n$ lead to larger $n^*$. In the case where $R(n)$ does not scale with the dataset size $n$ (i.e., $R(n) \in \Theta(1)$), $\|\bm u_n\|_2^2$ is always increasing; therefore, a dataset of size $T$ is recommended.

\subsection{Removal of Irrelevant Observations} \label{sec:main-stale}

In the previous section, we used a FIFO policy to study the joint effect of the temporal kernel $k_T$, the response time $R(n)$, and the maximal dataset size $n$. This analysis led to Recommendation~1. The simplicity of the FIFO policy makes it an analytically tractable reference, but it might be suboptimal because it needs not retain the most informative observations.

Given the maximal dataset size prescribed by Recommendation~1, the remaining question is therefore to select the observation to remove whenever this size is exceeded. \citet{bardou2024too} propose to select $\bm{o}_i \in \mathcal{D}$ that minimizes the integrated 2-Wasserstein distance~\citep{kantorovich1960mathematical} between the GP surrogate $\mathcal{GP}_{\mathcal{D}}\left(\mu_{\mathcal{D}},\sigma_{\mathcal{D}}^2\right)$ conditioned on $\mathcal{D}$ and the surrogate $\mathcal{GP}_{\widetilde{\mathcal{D}}}\left(\mu_{\widetilde{\mathcal{D}}},\sigma_{\widetilde{\mathcal{D}}}^2\right)$ conditioned on $\widetilde{\mathcal{D}}=\mathcal{D}\setminus\{\bm{o}_i\}$. This policy achieves good empirical performance~\citep{bardou2024too}, but its connection to the subsequent optimization decisions has not been formally established. In this section, we show that minimizing this distance controls the perturbation of the acquisition function induced by removing an observation, thereby providing a theoretical justification for the policy under the following assumption.

\begin{assumption}[Lipschitz Acquisition] \label{ass:acquisition_lipschitz}
    At iteration $T \in \mathbb{N}$, the acquisition function $\alpha : \mathcal{S} \times \mathcal{T} \to \mathbb{R}$ is built from the posterior mean $\mu(\bm x, t)$ and the posterior standard deviation $\sigma\left(\bm x, t\right)$ of the GP surrogate, that is, $\alpha(\bm x, t) = g_T(\mu(\bm x, t), \sigma\left(\bm x, t\right))$ for some function $g_T : \mathbb{R}^2 \to \mathbb{R}$. Furthermore, $g_T$ is a Lipschitz continuous function, that is, 
    \begin{equation} \label{eq:acquisition_lipschitz}
        |g_T(\bm u) - g_T(\bm v)| \leq L_T ||\bm u - \bm v||_2, \text{ }\forall \bm u, \bm v \in \mathbb{R}^2.
    \end{equation}
\end{assumption}

Assumption~\ref{ass:acquisition_lipschitz} holds for many acquisition functions. For example, if $\alpha$ is GPUCB~\citep{gpucb}, that is, $\alpha(\bm x, t) = \mu(\bm x, t) + \beta_T^{1/2} \sigma(\bm x, t)$, then Assumption~\ref{ass:acquisition_lipschitz} holds with $L_T = \sqrt{1 + \beta_T}$.

\begin{theorem} \label{thm:acquisition_wasserstein}
    Let $\alpha$ be an acquisition function satisfying Assumption~\ref{ass:acquisition_lipschitz}. Let us denote by $\alpha_\mathcal{D}$ (resp., $\alpha_{\Tilde{\mathcal{D}}}$) the acquisition function $\alpha$ exploiting the surrogates $\mathcal{GP}_\mathcal{D}$ (resp., $\mathcal{GP}_{\Tilde{\mathcal{D}}}$). Then, on any subset $\mathcal{S}' \times \mathcal{T}'$ of the spatio-temporal domain $\mathcal{S} \times \mathcal{T}$:
    \begin{equation} \label{eq:upper_bound_perturbation}
        ||\alpha_\mathcal{D} - \alpha_{\Tilde{\mathcal{D}}}||_2 \leq L_T W_2\left(\mathcal{GP}_\mathcal{D}, \mathcal{GP}_{\Tilde{\mathcal{D}}}\right),
    \end{equation}
    where $||\alpha_\mathcal{D} - \alpha_{\Tilde{\mathcal{D}}}||_2$ is the $L^2$ distance between $\alpha_\mathcal{D}$ and $\alpha_{\Tilde{\mathcal{D}}}$ on $\mathcal{S}' \times \mathcal{T}'$ and $W_2\left(\mathcal{GP}_\mathcal{D}, \mathcal{GP}_{\Tilde{\mathcal{D}}}\right)$ is the integrated 2-Wasserstein distance between the Gaussian marginals of $\mathcal{GP}_\mathcal{D}$ and $\mathcal{GP}_{\Tilde{\mathcal{D}}}$ on $\mathcal{S}' \times \mathcal{T}'$.
\end{theorem}

Theorem~\ref{thm:acquisition_wasserstein} is proven in Appendix~\ref{app:recommended_stale}. Assumption~\ref{ass:acquisition_lipschitz} allows to upper bound the $L^2$ distance between $\alpha_{\mathcal{D}}$ and $\alpha_{\Tilde{\mathcal{D}}}$ on any subset of $\mathcal{S} \times \mathcal{T}$ in terms of the posterior means and variances of $\mathcal{GP}_\mathcal{D}$ and $\mathcal{GP}_{\Tilde{\mathcal{D}}}$, which in turn naturally leads to the integrated Wasserstein distance between the GP surrogates.

Theorem~\ref{thm:acquisition_wasserstein} provides an interesting connection between the discrepancy of the acquisition functions built from two surrogate models $\mathcal{GP}_\mathcal{D}$ and $\mathcal{GP}_{\tilde{\mathcal{D}}}$, and the integrated 2-Wasserstein distance between these two surrogates. This connection motivates the following stale data management policy: given a maximal dataset size $n$ and a dataset $\mathcal{D}$ of size $T$ (where $T > n$), one should build $\tilde{\mathcal{D}} = \mathcal{D} \setminus X^*$, where $X^* = \argmin_{X \in 2^\mathcal{D}, |X| = T - n} W_2(\mathcal{GP}_\mathcal{D}, \mathcal{GP}_{\mathcal{D} \setminus X})$. Unfortunately, this policy would become intractable in the long run. Therefore, to deal with stale data, we recommend the same greedy heuristic as the one introduced in~\citet{bardou2024too}, which can be run online.

\paragraph{Recommendation 2.}
Given $\mathcal{D}=\{\bm o_1,\ldots,\bm o_n\}$, let $\widetilde{\mathcal{D}}_i=\mathcal{D}\setminus\{\bm o_i\}$ and remove $\bm o_{i^*}$, where $i^*=\argmin_{i\in[n]} W_2(\mathcal{GP}_{\mathcal{D}},\mathcal{GP}_{\widetilde{\mathcal{D}}_i})$.

\section{\AlgoName: Bayesian Optimization in the Long Term} \label{sec:algo}

In the previous section, we have made two recommendations about the dataset size and the observation removal policy for TVBO algorithms, which are only useful when the optimization task is a long-horizon problem (see Section~\ref{sec:background-regimes}), so that removing observations from $\mathcal{D}$ produces a noticeable effect on the response time $R(n)$ and, therefore, on performance.

We now introduce \AlgoName{}, a TVBO algorithm implementing both recommendations (Algorithm~\ref{alg:dbo}). At each iteration, it selects a query $(\bm x,t)$, observes its response $f(\bm x, t)$, and adds the observation to $\mathcal{D}$. Recommendation~1 determines the maximal dataset size from $k_T$ and $R(n)$, while Recommendation~2 selects which observations to discard if removal is needed.

\AlgoName{} is based on W-DBO~\citep{bardou2024too}, but it replaces its user-specified removal budget with the maximal dataset size derived from Theorem~\ref{thm:upper_regret_bound}. Through its dependence on $R(n)$, this size~\eqref{eq:max_dataset_size} adapts to differences in response times (caused by, e.g., different computing powers).

\begin{algorithm}[tb]
   \caption{\AlgoName}
   \label{alg:dbo}
\begin{algorithmic}
   \STATE {\bfseries Input:} objective $f : \mathcal{S} \times \mathcal{T} \to \mathbb{R}$, acquisition function $\alpha$, clock $\mathcal{C}$
   \STATE Init dataset $\mathcal{D} = \emptyset$
   \WHILE{\textbf{true}}
       \STATE Get current time $t$ from $\mathcal{C}$
       \STATE Find $x_t = \argmax_{\bm x \in \mathcal{S}} \alpha_\mathcal{D}(\bm x, t)$
       \STATE Observe $y = f(\bm x_t, t) + \epsilon$
       \STATE $\mathcal{D} = \mathcal{D} \cup \left\{((\bm x_t, t), y)\right\}$
       \STATE Infer hyperparams and response time $R(n)$ from data
       \STATE Find the recommended $n^*$ with~\eqref{eq:max_dataset_size}
       \IF{$|\mathcal{D}| > n^*$}
          \STATE Find $\bm o^* = \argmin_{\bm o \in \mathcal{D}} W_2\left(\mathcal{GP}_{\mathcal{D}}, \mathcal{GP}_{\mathcal{D} \setminus \left\{\bm o\right\}}\right)$
          \STATE $\mathcal{D} = \mathcal{D} \setminus \left\{\bm o^*\right\}$
       \ENDIF
   \ENDWHILE
\end{algorithmic}
\end{algorithm}

\paragraph{Estimating the Response Time.}
Algorithm~\ref{alg:dbo} estimates $R(n)$ online, similarly to the GP hyperparameters (variance $\lambda$, spatial and temporal lengthscales, and noise variance $\sigma_0^2$). At each iteration, the response time $R(|\mathcal{D}|)$ is paired with the dataset size $|\mathcal{D}|$ and recorded in a dataset $\mathcal{R}$. Motivated by the $\mathcal{O}(n^3)$ complexity of GP inference, \AlgoName{} fits a cubic polynomial to $\mathcal{R}$ to estimate $R(n)$.

\section{Numerical Results} \label{sec:results}

In this section, we evaluate \AlgoName\ against the state-of-the-art of TVBO. All algorithms in the TVBO literature that use an infinite variational budget for the objective function $f$ are included: ABO~\citep{nyikosa2018bayesian}, ET-GPUCB~\citep{brunzemaevent}, W-DBO~\citep{bardou2024too} (which implements Recommendation~2 but not Recommendation~1), TV-GPUCB and R-GPUCB~\citep{bogunovic2016time}. We also include the vanilla GPUCB algorithm~\citep{gpucb} and FIFO (which implements Recommendation~1 but not Recommendation~2).

\subsection{Experimental Setting} \label{sec:results-setting}

We evaluate the TVBO algorithms on a set of 10 synthetic and real-world benchmarks, which are described in detail in Appendix~\ref{app:benchmarks}. Some can be viewed as either quasi-static or as having a near-constant response time (see Section~\ref{sec:background-regimes} for definitions of these types of problems). Each benchmark is a $(d+1)$-dimensional function to optimize. The first $d$ dimensions make up the spatial domain $\mathcal{S}$, which is normalized in $[0, 1]^d$. The $(d+1)$th dimension is the time dimension.

Each experiment starts with a warm-up phase consisting of 15 random queries. Then, at each iteration, the TVBO algorithms must perform the following tasks in real-time:
\begin{itemize}
    \item \textbf{Hyperparameters inference}: the variance $\lambda$, the spatial lengthscale $l_S$ and the observational noise $\sigma^2_0$ must always be inferred. When applicable, the temporal lengthscale $l_T$ and the response time $R(n)$ of the algorithm are also inferred at this stage. Unless requested otherwise by the algorithm, each solution uses a Matérn-5/2 spatial covariance function, and a Matérn-3/2 temporal covariance function.
    \item \textbf{Optimization of the acquisition function}: the acquisition function is always GPUCB~\citep{gpucb}, and the next observation $\bm x_t$ is found by using multi-start gradient ascent.
    \item \textbf{Function observation}: a noisy function value $y_i = f(\bm x_i, t_i) + \epsilon$ is observed, and the observation $(\bm x_i, t_i, y_i)$ is added to the dataset $\mathcal{D}$.
    \item \textbf{Dataset cleaning}: this is the stage in which some observations in $\mathcal{D}$ might be removed.
\end{itemize}

Each TVBO algorithm has been implemented with the same BO framework, namely BOTorch~\citep{balandat2020botorch}. The critical, time-consuming stages (i.e., hyperparameters inference and acquisition function optimization) are performed using the same BOTorch routines. Finally, all experiments have been independently replicated 10 times on a laptop equipped with an Intel Core i9-9980HK @ 2.40 GHz with 8 cores (16 threads).

\subsection{Experimental Results} \label{sec:results-all}

\begin{figure}[t]
    \centering
    \includegraphics[height=5.3cm]{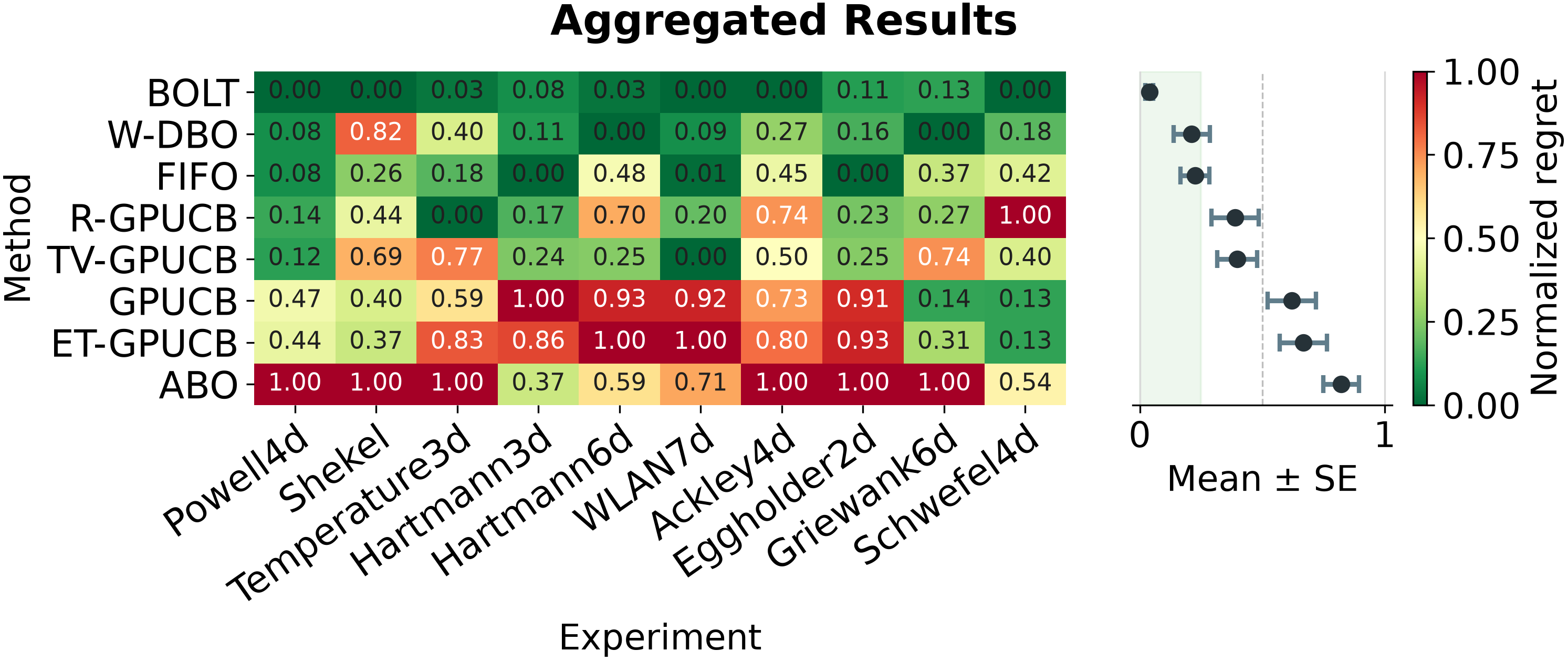}
    \caption{(Left)~Heatmap of normalized regret (lower is better) for each benchmark and each method. For each benchmark, the best performing TVBO algorithm gets a normalized regret of 0, and the worst performing TVBO algorithm gets a normalized regret of 1. (Right)~Averaged normalized regret across the benchmarks (lower is better).}
    \label{fig:agg_perf}
\end{figure}

Table~\ref{tab:eval} reports average regrets and standard errors for each algorithm and benchmark (Appendix~\ref{app:benchmarks}), while Figure~\ref{fig:agg_perf} summarizes aggregate performance.

Figure~\ref{fig:ex_dataset_evolution} illustrates how \AlgoName{} adapts its dataset size on the synthetic benchmarks Eggholder and Powell. On Eggholder (left), it discards many observations, consistent with the short temporal correlations of this benchmark. Algorithms that remove observations (R-GPUCB, W-DBO, FIFO, and \AlgoName{}) consequently perform well. On Powell (right), observations remain informative over longer periods, and \AlgoName{} retains a growing dataset, similarly to TV-GPUCB. Despite their comparable dataset sizes, \AlgoName{} achieves substantially lower regret (Table~\ref{tab:eval}). A possible explanation is its more flexible temporal model under Assumptions~\ref{ass:gp}, \ref{ass:separability}, and~\ref{ass:kernel_props}: GPUCB ignores temporal variation, whereas TV-GPUCB assumes Markovian dynamics~\citep{bogunovic2016time}.

\begin{figure}[t]
    \centering
    \includegraphics[height=4.5cm]{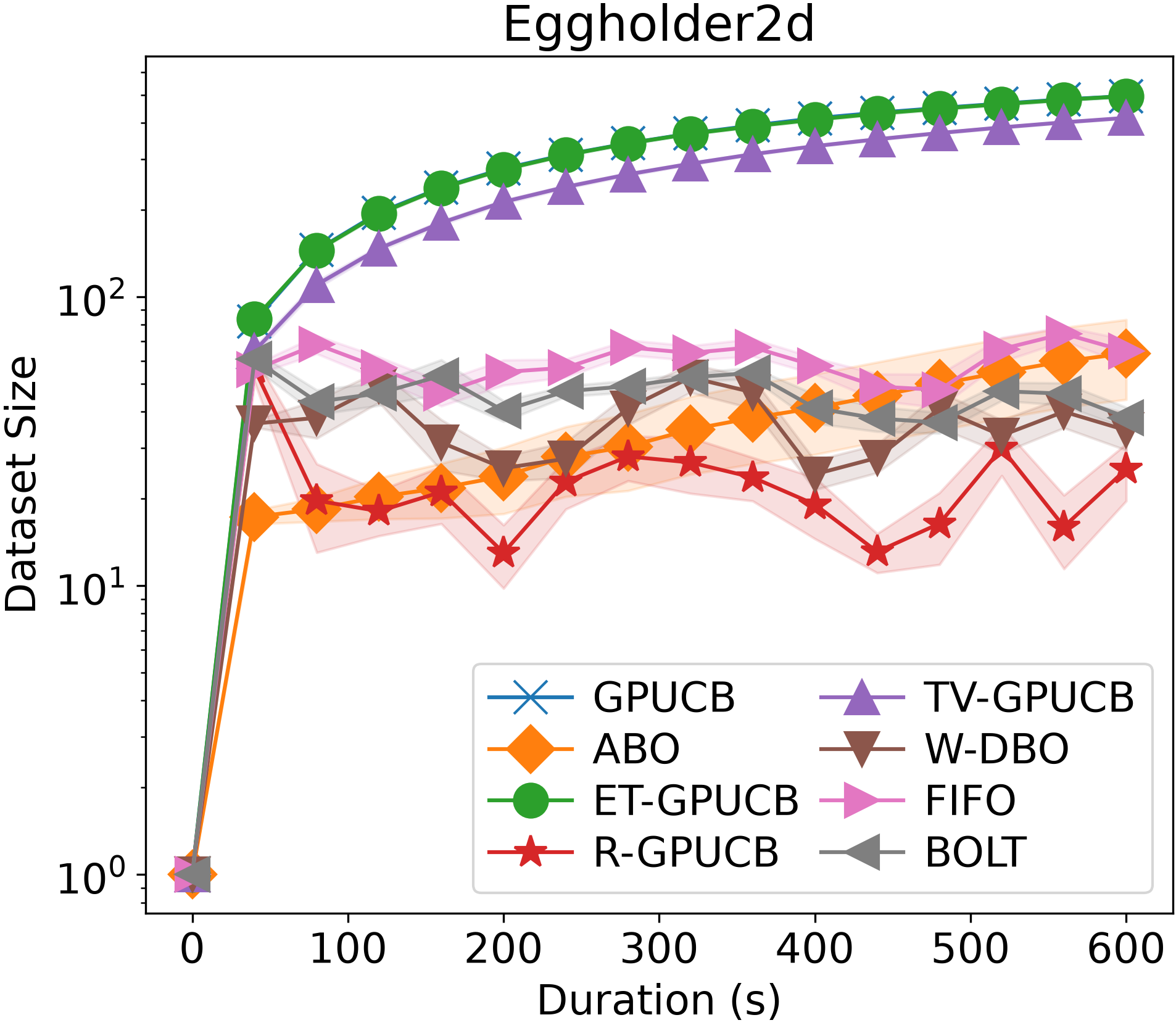}
    \includegraphics[height=4.5cm]{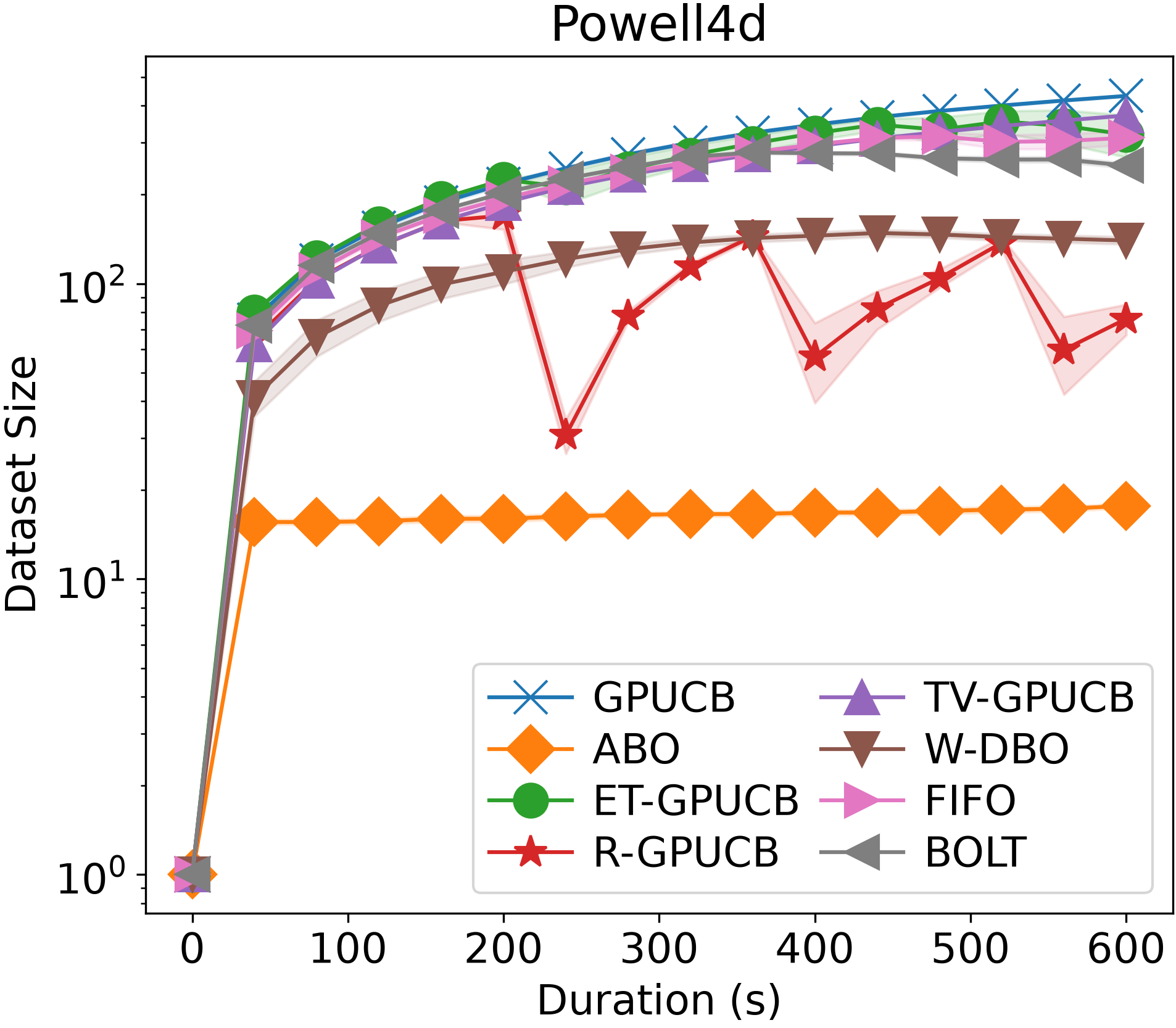}
    \caption{Dataset sizes on Eggholder (left) and Powell (right), with logarithmic y-axes.}
    \label{fig:ex_dataset_evolution}
\end{figure}

The benchmarks also cover quasi-static problems (e.g., Temperature) and near-constant response times (e.g., Shekel), as listed in Table~\ref{tab:parameters} of Appendix~\ref{app:benchmarks}. Performance is broadly similar in the near-constant regime, apart from ET-GPUCB, whose difficulties with online hyperparameter inference are discussed by~\citet{brunzemaevent}. On quasi-static problems, methods that remove observations (R-GPUCB, W-DBO, FIFO, and \AlgoName{}) perform best. Across these regimes and the remaining benchmarks, \AlgoName{} consistently achieves competitive or superior performance (Figure~\ref{fig:agg_perf}). This does not come at the cost of a significantly larger response time, which stays in the same range as for W-DBO, FIFO, and TV-GPUCB (Figures~\ref{fig:schwefel}--\ref{fig:wlan}, Appendix~\ref{app:benchmarks}).

\section{Conclusion}

In this paper, we developed a regret analysis that explicitly accounts for the dependence of computational overhead and observation sampling frequency on the dataset size of a TVBO algorithm. Our main result, Theorem~\ref{thm:upper_regret_bound}, establishes an asymptotic upper regret bound for GPUCB applied to a spatio-temporal GP with a FIFO policy. To the best of our knowledge, it is the first result to jointly relate TVBO regret to the temporal covariance structure and the algorithm's response time, yielding a principled recommendation for the maximal dataset size. Our second result, Theorem~\ref{thm:acquisition_wasserstein}, bounds the acquisition-function perturbation induced by observation removal using the integrated Wasserstein distance between GP surrogates, providing theoretical support for the policy of~\citet{bardou2024too}. Together, these results inform two complementary decisions: how many observations to retain and which to remove. We combined these recommendations in \AlgoName{}, which consistently achieves competitive or superior performance across diverse synthetic and real-world benchmarks.

Future work will extend this analysis beyond Assumptions~\ref{ass:separability} and~\ref{ass:kernel_props}. In particular, nonseparable spatio-temporal kernels could capture dynamics whose temporal behavior varies across the spatial domain and could be applied to more complex spatio-temporal black-box optimization problems.

\bibliography{main}
\bibliographystyle{preprint}

\newpage

\appendix

\section{Upper Regret Bound without a Metronome} \label{app:dataset_size}

In this appendix, we provide all the details required to prove Theorem~\ref{thm:upper_regret_bound}. For the sake of completeness, we start by deriving the usual instantaneous regret bound provided in most regret proofs that involve GPUCB~\citep{gpucb, bogunovic2016time}.

\begin{lemma} \label{lem:immediate_regret_bound}
    Let $r_{n} = f(\bm x^*_{n}, t_n) - f(\bm x_{n}, t_n)$ be the instantaneous regret at the $n$-th iteration of the TVBO algorithm $\mathcal{A}$, where $\bm x^*_{n} = \argmax_{\bm x \in \mathcal{S}} f(\bm x, t_n)$, $\bm x_{n} = \argmax_{\bm x \in \mathcal{S}} \alpha_\mathcal{D}(\bm x, t_n)$, $\alpha_\mathcal{D}$ is the GPUCB acquisition function computed on $\mathcal{D}$ and where $\mathcal{D}$ is a dataset of observations collected with the distribution $\mu$. Pick $\delta \in (0, 1)$, then with probability at least $1 - \delta$,
    \begin{equation}
        r_n \leq 2 \beta_{n}^{1/2} \sigma_\mathcal{D}(\bm x_{n}, t_n) + \frac{1}{n^2}
    \end{equation}
    where $\sigma_\mathcal{D}(\bm x, t)$ is the posterior standard deviation (see~\eqref{eq:post_var}) and where
    \begin{equation} \label{eq:beta}
        \beta_n = 2d\log\left(b d n^2\sqrt{\log(da\pi^2n^2/3\delta)} / 6\delta\right) + 4\log(\pi n),
    \end{equation}
    with $a, b > 0$ given by Assumption~\ref{ass:lipschitz}.
\end{lemma}

\begin{proof}
    For the sake of consistency with the literature, we will reuse, whenever appropriate, the notations in~\citet{bogunovic2016time}.
    
    Let us set a discretization $\mathcal{S}_n$ of the spatial domain $\mathcal{S} \subseteq [0, 1]^d$. $\mathcal{S}_n$ is of size $\tau_n^d$ and satisfies
    \begin{equation} \label{eq:grid_distance}
        ||\bm x - [\bm x]_n||_1 \leq \frac{d}{\tau_n}, \text{ } \forall \bm x \in \mathcal{S},
    \end{equation}
    where $[\bm x]_n = \argmin_{\bm s \in \mathcal{S}_n} ||\bm s - \bm x||_1$ is the closest point in $\mathcal{S}_n$ to $\bm x$. Note that letting $S_n$ be a uniform grid with discretization step $1/\tau_n$ on each dimension of $\mathcal{S}$ ensures~\eqref{eq:grid_distance}.

    Let us now fix $\delta > 0$ and condition on a high-probability event. If $\beta_n = 2 \log \frac{\tau_n^d \pi^2 n^2}{3 \delta}$, then
    \begin{equation} \label{eq:concentration_bound}
        |f(\bm x, t_n) - \mu_{\mathcal{D}}(\bm x, t_n)| \leq \beta_n^{1/2} \sigma_{\mathcal{D}}(\bm x, t_n), \text{ } \forall n \in \mathbb{N}, \forall \bm x \in \mathcal{S}_n
    \end{equation}
    with probability at least $1 - \delta/2$. This directly comes from $f(\bm x, t_n) \sim \mathcal{N}\left(\mu_{\mathcal{D}}(\bm x, t_n), \sigma^2_{\mathcal{D}}(\bm x, t_n)\right)$ and from the Chernoff concentration inequality applied to Gaussian tails $\mathbb{P}\left(|f(\bm x, t_n) - \mu_\mathcal{D}(\bm x, t_n)| \leq \sqrt{\beta_n} \sigma_{\mathcal{D}}(\bm x, t_n)\right) \geq 1 - e^{-\beta_n / 2}$. Therefore, our choice of $\beta_n$ ensures that for a given $n$ and a given $\bm x \in \mathcal{S}_n$, $|f(\bm x, t_n) - \mu_{\mathcal{D}}(\bm x, t_n)| \leq \beta_n^{1/2} \sigma_{\mathcal{D}}(\bm x, t_n)$ occurs with a probability at least $1 - \frac{3 \delta}{\pi^2 n^2 \tau_n^d}$. A union bound taken over $n \in \mathbb{N}$ and $\bm x \in \mathcal{S}_n$ establishes~\eqref{eq:concentration_bound} with probability $1 - \delta/2$.

    In particular, setting $\tau_n = Ldn^2$, where $L$ is defined in Assumption~\ref{ass:lipschitz}, gives that for all $\bm x \in \mathcal{S}$ and all $t_n \in \mathcal{T}$,
    \begin{equation} \label{eq:distance_fvalues_discretized}
        |f(\bm x, t_n) - f([\bm x]_n, t_n)| \leq \frac{1}{n^2}
    \end{equation}
    with probability at least $1 - \delta/2$.

    Because $\tau_n = Ldn^2$, $\beta_n$ becomes $2d\log\left(b\sqrt{\log(da\pi^2n^2/3\delta)}dn^2 / 6\delta\right) + 4\log(\pi n)$. Using the triangle inequality and combining~\eqref{eq:concentration_bound} with~\eqref{eq:distance_fvalues_discretized}, we get that with probability $1 - \delta$, $\bm x^*_{n} = \argmax_{\bm x \in \mathcal{S}} f(\bm x, t_n)$ satisfies
    \begin{align}
        |f(\bm x^*_{n}, t_n) - \mu_{\mathcal{D}}([\bm x^*_{n}]_n, t_n)| &\leq |f([\bm x^*_{n}]_n, t_n) - \mu_{\mathcal{D}}([\bm x^*_{n}]_n, t_n)| + |f(\bm x^*_{n}, t_n) - f([\bm x^*_{n}]_n, t_n)| \nonumber\\
        &\leq \beta_n^{1/2} \sigma_{\mathcal{D}}([\bm x^*_{n}]_n, t_n) + \frac{1}{n^2}. \label{eq:bound_fvalue_mean_discretized}
    \end{align}
    We can now upper bound the instantaneous regret of the TVBO algorithm $\mathcal{A}$:
    \begin{align}
        r_{n} &= f(\bm x^*_{n}, t_n) - f(\bm x_{n}, t_n) \nonumber\\
        &\leq \mu_{\mathcal{D}}([\bm x^*_{n}]_n, t_n) + \beta_n^{1/2} \sigma_{\mathcal{D}}([\bm x^*_{n}]_n, t_n) + \frac{1}{n^2} - f(\bm x_{n}, t_n) \label{eq:proof_discretized}\\
        &\leq \mu_{\mathcal{D}}(\bm x_{n}, t_n) + \beta_n^{1/2} \sigma_{\mathcal{D}}(\bm x_{n}, t_n) + \frac{1}{n^2} - f(\bm x_{n}, t_n) \label{eq:proof_acq_maximized}\\
        &\leq 2 \beta_n^{1/2} \sigma_{\mathcal{D}}(\bm x_{n}, t_n) + \frac{1}{n^2} \label{eq:proof_concentration},
    \end{align}
    where~\eqref{eq:proof_discretized} follows directly from~\eqref{eq:bound_fvalue_mean_discretized}, \eqref{eq:proof_acq_maximized} from the definition of $\bm x_{n} = \argmax_{\bm x \in \mathcal{S}} \alpha_\mathcal{D}(\bm x, t_n) = \argmax_{\bm x \in \mathcal{S}} \mu_{\mathcal{D}}(\bm x, t_n) + \beta_n^{1/2} \sigma_{\mathcal{D}}(\bm x, t_n)$ and~\eqref{eq:proof_concentration} from~\eqref{eq:concentration_bound}.
\end{proof}

Lemma~\ref{lem:immediate_regret_bound} shows that the posterior variance of the surrogate GP~\eqref{eq:post_var} plays a key role in the regret bound. Let us now derive a first bound on the cumulative regret $R_T$, by extending an original idea from~\citet{gpucb} to the time-varying setting.

\begin{lemma} \label{lem:first_generic_regret_bound}
    Let $\{\mathcal{D}_i\}_{i = 1, \cdots, T}$ be a sequence of datasets of size $|\mathcal{D}_i| = \min(i, n)$ for $i = 1, \cdots, T$. Let $R_T = \sum_{i =1}^T r_i$ be the cumulative regret of a TVBO algorithm with maximal dataset size $n$. Then,
    \begin{equation} \label{eq:first_generic_regret_bound}
        R_T \leq \sqrt{C_1T\beta_T\left(\gamma_n + \frac{1}{2} \sum_{i = n+1}^T \log(1 + \sigma^{-2}_0 \sigma^2_{\mathcal{D}_i}(\bm x_i, t_i))\right)} + \frac{\pi^2}{6}
    \end{equation}
   where
    \begin{equation} \label{eq:c1_definition}
        C_1 = \frac{8}{\log(1 + \sigma^{-2}_0)},
    \end{equation}
    and where $\gamma_n$ is the mutual information $\gamma_n = \frac{1}{2} \sum_{i = 1}^n \log(1 + \sigma^{-2}_0 \sigma_{\mathcal{D}_i}^2(\bm x_i, t_i))$.
\end{lemma}

\begin{proof}
    We have
    \begin{align}
        R_T &= \sum_{i = 1}^T r_i \nonumber\\
        &\leq 2 \left(\sum_{i = 1}^n \beta_i^{1/2} \sigma_{\mathcal{D}_i}(\bm x_i, t_i) + \sum_{i = n+1}^T \beta_i^{1/2} \sigma_{\mathcal{D}_i}(\bm x_i, t_i)\right) + \frac{\pi^2}{6} \label{eq:proof_use_immediate_reg_bound}\\
        &\leq \sqrt{4 T \beta_T \left(\sum_{i = 1}^n \sigma^2_{\mathcal{D}_i}(\bm x_i, t_i) + \sum_{i = n+1}^T \sigma^2_{\mathcal{D}_i}(\bm x_i, t_i)\right)} + \frac{\pi^2}{6} \label{eq:proof_use_cauchyschwarz}
    \end{align}
    where~\eqref{eq:proof_use_immediate_reg_bound} uses Lemma~\ref{lem:immediate_regret_bound}, the solution to the Basel problem $\sum_{i = 1}^T i^{-2} \leq \sum_{i = 1}^\infty i^{-2} = \pi^2/6$ and the maximal dataset size $n$ of the TVBO algorithm, while~\eqref{eq:proof_use_cauchyschwarz} uses the Cauchy-Schwarz inequality and $\beta_i \leq \beta_T$ for any $i \leq T$.

    Going further, we get
    \begin{align}
        R_T &\leq \sqrt{4\sigma^2_0 T \beta_T \left(\sum_{i = 1}^n \sigma^{-2}_0 \sigma^2_{\mathcal{D}_i}(\bm x_i, t_i) + \sum_{i = n+1}^T \sigma^{-2}_0 \sigma^2_{\mathcal{D}_i}(\bm x_i, t_i)\right)} + \frac{\pi^2}{6} \nonumber\\
        &\leq \sqrt{\frac{8}{\log(1 + \sigma^{-2}_0)} T \beta_T \left(\frac{1}{2}\sum_{i = 1}^n \log(1 + \sigma^{-2}_0 \sigma^2_{\mathcal{D}_i}(\bm x_i, t_i)) + \frac{1}{2}\sum_{i = n+1}^T \log(1 + \sigma^{-2}_0 \sigma^2_{\mathcal{D}_i}(\bm x_i, t_i))\right)} + \frac{\pi^2}{6} \label{eq:proof_use_log_identity}\\
        &= \sqrt{C_1 T \beta_T \left(\gamma_n + \frac{1}{2} \sum_{i = n+1}^T \log(1 + \sigma^{-2}_0 \sigma^2_{\mathcal{D}_i}(\bm x_i, t_i))\right)} + \frac{\pi^2}{6} \label{eq:proof_use_c1_gamma_def}
    \end{align}
    where~\eqref{eq:proof_use_log_identity} uses the identity $z^2 \leq \sigma^{-2}_0 \log(1 + z^2)/\log(1 + \sigma^{-2}_0)$ for $z^2 \in [0, \sigma^{-2}_0]$ and where~\eqref{eq:proof_use_c1_gamma_def} uses~\eqref{eq:c1_definition} and the definition of mutual information $\gamma_n$.
\end{proof}

We now provide an upper bound for the posterior variance $\sigma^2_{\mathcal{D}}(\bm x, t)$.

\begin{lemma} \label{lem:bound_on_variance}
    Let $\mathcal{A}$ be a TVBO algorithm, $\mathcal{D}$ be its dataset of observations with fixed maximal size $n$ and $R(n)$ be its response time. Let
    \begin{equation} \label{eq:def_G}
        G = \sup_{m \geq 1} \frac{1}{R(m)} \int_0^{mR(m)} g_T(\tau) d\tau,
    \end{equation}
    where $g_T$ is the function upper-bounding $k_T$ in Assumption~\ref{ass:kernel_props}. Then,
    \begin{equation} \label{eq:post_cov_upper_bound}
        \sigma_{\mathcal{D}}^2(\bm x, t) \leq 1 - C_2 \|\bm u_n\|_2^2,
    \end{equation}
    where $C_2 = \min_{\bm x, \bm x' \in \mathcal{S}} k_S^2(\bm x, \bm x') / (1 + 2G + \sigma^2_0)$ and where
    \begin{equation*}
        \bm u_n = \left(k_T(R(n)), k_T(2R(n)), \cdots, k_T(nR(n))\right).
    \end{equation*}
\end{lemma}

\begin{proof}
    Let us start by recalling the definition of the posterior variance in~\eqref{eq:post_var}:
    \begin{align}
        \sigma^2_{\mathcal{D}}(\bm x, t) &= k((\bm x, t), (\bm x, t)) - \bm k^\top((\bm x, t), \mathcal{D}) \bm \Delta^{-1} \bm k((\bm x, t), \mathcal{D}) \nonumber\\
        &= 1 - \bm k^\top((\bm x, t), \mathcal{D}) \bm \Delta^{-1} \bm k((\bm x, t), \mathcal{D}) \label{eq:proof_alternative_post_var}
    \end{align}
    where~\eqref{eq:proof_alternative_post_var} follows from Assumptions~\ref{ass:separability} and~\ref{ass:kernel_props}.

    Finding an upper bound on $\sigma^2_{\mathcal{D}}(\bm x, t)$ boils down to finding a lower bound on the quadratic form $q_n(\bm x, t) = \bm k^\top((\bm x, t), \mathcal{D}) \bm \Delta^{-1} \bm k((\bm x, t), \mathcal{D})$. Let us first consider the kernel vector $\bm k((\bm x, t), \mathcal{D})$. Because the dataset $\mathcal{D} = \left\{\left(\bm x_1, t_1, y_1\right), \cdots, \left(\bm x_n, t_n, y_n\right)\right\}$ has a fixed maximal dataset size $n$, for every iteration $m > 2n$:
    \begin{equation} \label{eq:proof_explicit_kernel_vector}
        \bm k^\top((\bm x, t), \mathcal{D}) = \left(k_S(\bm x, \bm x_n) k_T(R(n)), \cdots, k_S(\bm x, \bm x_1) k_T(nR(n))\right).
    \end{equation}

    The expression in~\eqref{eq:proof_explicit_kernel_vector} is due to $\mathcal{A}$ having a fixed dataset size $n$, and therefore a constant response time $R(n)$. In such a dataset $\mathcal{D} = \left\{(\bm x_i, t_i, y_i)\right\}_{i \in [n]}$, the temporal input $t_i$ is deterministic. Because $\mathcal{A}$ follows a FIFO removal policy, it first fills its dataset during the first $n$ iterations. Once the maximal dataset size has been reached, its response time remains equal to $R(n)$, and the oldest observation is removed after every new query. Consequently, for every iteration $m>2n$, the dataset consists of the $n$ most recent observations, collected at constant intervals $R(n)$, and the kernel vector $\bm k^\top((\bm x, t), \mathcal{D})$ verifies~\eqref{eq:proof_explicit_kernel_vector}.

    Next, consider $\bm \Delta = \bm k(\mathcal{D}, \mathcal{D}) + \sigma^2_0 \bm I$. Because it is a positive definite matrix, its spectral decomposition $\bm Q \bm \Lambda \bm Q^\top$ exists, where $\bm Q = \left(\bm \phi_1, \cdots, \bm \phi_n\right)$ is the orthogonal matrix whose columns are the eigenvectors of $\bm \Delta$, and $\bm \Lambda = \text{diag}\left(\lambda_1 + \sigma^2_0, \cdots, \lambda_n + \sigma^2_0\right)$ is the diagonal matrix comprising the associated eigenvalues. Note that, by convention, the eigenvalues are sorted in decreasing order $\lambda_1 \geq \lambda_2 \geq \cdots \geq \lambda_n$.

    The quadratic form $q_n(\bm x, t)$ then becomes
    \begin{align}
        q_n(\bm x, t) &= \bm k^\top((\bm x, t), \mathcal{D}) \bm Q \bm \Lambda^{-1} \bm Q^\top \bm k((\bm x, t), \mathcal{D}) \nonumber\\
        &= \bm v^\top \bm \Lambda^{-1} \bm v \nonumber\\
        &= \sum_{i = 1}^n \frac{v_i^2}{\lambda_i + \sigma^2_0} \nonumber\\
        &\geq \frac{1}{\lambda_1 + \sigma^2_0} \sum_{i = 1}^n \sum_{j = 1}^n \sum_{k = 1}^n k((\bm x, t), (\bm x_j, t_j)) \phi_{ij} \phi_{ik} k((\bm x, t), (\bm x_k, t_k)) \label{eq:proof_quadratic_form_sum}
    \end{align}
    where $\bm v = \bm Q^\top \bm k((\bm x, t), \mathcal{D})$ while \eqref{eq:proof_quadratic_form_sum} is the quadratic form $q_n(\bm x, t)$ rewritten with sums instead of matrix products and using the fact that $\lambda_1 \geq \lambda_i$ for all $i \in [n]$. Furthermore, note that the largest eigenvalue $\lambda_1$ of the Gram matrix $\bm K$ built from observations collected every $R(n)$ unit of time can be bounded from above:
    \begin{align}
        \lambda_1 &\leq \max_{i \in [n]} \bm K_{ii} + \sum_{j \neq i} \left| \bm K_{ij} \right| \label{eq:proof_gershgorin}\\
        &\leq 1 + 2 \sum_{j = 1}^{n-1} \left| k_T(jR(n)) \right| \label{eq:proof_gram_diag_row}\\
        &\leq 1 + 2 \sum_{j = 1}^{n-1} g_T(jR(n)) \label{eq:proof_use_gt}\\
        &\leq 1 + \frac{2}{R(n)} \int_0^{(n-1)R(n)} g_T(\tau) d\tau \label{eq:proof_use_quad_rule}\\
        &\leq 1 + 2G \label{eq:proof_use_G}
    \end{align}
    where~\eqref{eq:proof_gershgorin} is an application of Gershgorin's theorem for real and symmetric matrices, \eqref{eq:proof_gram_diag_row}~holds because all the diagonal entries of $\bm K$ are equal to one, and because for all $j \in [n]$, there are at most two entries $k_T(jR(n))$ at each temporal lag $jR(n)$, and \eqref{eq:proof_use_gt} holds under Assumption~\ref{ass:kernel_props}. Inequality~\eqref{eq:proof_use_quad_rule} is an application of the quadrature rule on the truncated range $[0, (n-1)R(n)]$: because $g_T$ is nonnegative and nonincreasing, $R(n) g_T(jR(n)) \leq \int_{(j-1)R(n)}^{jR(n)} g_T(\tau) d\tau$ for every $j \geq 1$, and summing these inequalities over $j = 1, \cdots, n-1$ yields~\eqref{eq:proof_use_quad_rule}. Finally, \eqref{eq:proof_use_G} holds because $g_T$ is nonnegative, so that $\int_0^{(n-1)R(n)} g_T(\tau) d\tau \leq \int_0^{nR(n)} g_T(\tau) d\tau$, and because of the definition~\eqref{eq:def_G} of $G$.

    Let $\bar \lambda = 1 + 2G$. Then, rearranging the sums in~\eqref{eq:proof_quadratic_form_sum}, we have
    \begin{align}
        q_n(\bm x, t) &\geq \frac{1}{\bar \lambda + \sigma^2_0} \sum_{j = 1}^n \sum_{k = 1}^n k((\bm x, t), (\bm x_j, t_j)) k((\bm x, t), (\bm x_k, t_k)) \underbrace{\sum_{i = 1}^n \phi_{ij} \phi_{ik}}_{\delta_{jk}} \nonumber\\
        &= \frac{1}{\bar \lambda + \sigma^2_0} \sum_{i = 1}^n k^2((\bm x, t), (\bm x_i, t_i)) \label{eq:proof_orthogonality_qf}\\
        &= \frac{1}{\bar \lambda + \sigma^2_0} \sum_{i = 1}^n k_S^2(\bm x, \bm x_i) k_T^2(t, t_i) \label{eq:proof_qf_kernel_expansion}
    \end{align}
    where $\delta_{jk}$ is the Kronecker delta with value $1$ if $j = k$ and $0$ otherwise, \eqref{eq:proof_orthogonality_qf} is due to the orthonormality of the eigenvectors of $\bm K$ and \eqref{eq:proof_qf_kernel_expansion} is due to Assumption~\ref{ass:separability}.

    Let $\kappa = \min_{\bm x, \bm x' \in \mathcal{S}} k_S(\bm x, \bm x')$. Then, $\kappa > 0$ as per Assumption~\ref{ass:kernel_props} and we finally have
    \begin{align*}
        q_n(\bm x, t) &\geq \frac{\kappa^2}{\bar \lambda + \sigma^2_0} \sum_{i = 1}^n k_T^2(t, t_i)\\
        &= C_2 \|\bm u_n\|^2_2,
    \end{align*}
    where $C_2 = \kappa^2 / (\bar \lambda + \sigma^2_0)$.

    Noting that $\sigma^2_\mathcal{D}(\bm x, t) = 1 - q_n(\bm x, t) \leq 1 - C_2 \|\bm u_n\|_2^2$ concludes the proof. Recalling that $\sigma^2_{\mathcal{D}}(\bm x, t) \geq 0$ also yields that $0 \leq 1 - C_2 \|\bm u_n\|_2^2$, confirming that $1 - C_2 \|\bm u_n\|_2^2$ is non-negative.
\end{proof}

\begin{remark} \label{rem:G_vs_integrability}
    The constant $G$ can be finite even when $g_T$ is not integrable. Suppose that $g_T(\tau) \leq c(1 + \tau)^{-p}$ and $R(m) \geq c_R m^q$ for some $c, c_R, p > 0$ and $q \geq 0$. Then $G < \infty$ whenever $p(1 + q) > 1$, a condition that does not require $p > 1$. In particular, temporal correlation functions whose tails decay as $\Theta(\tau^{-1})$, such as sinc or a Rational Quadratic with $\alpha = 1/2$, satisfy it for every $q > 0$, although no integrable $g_T$ dominates them (see Appendix~\ref{app:sufficient_cond} for a detailed discussion on these two cases).
\end{remark}

Finally, let us prove Theorem~\ref{thm:upper_regret_bound}.

\begin{proof}
    Combining Lemmas~\ref{lem:first_generic_regret_bound} and~\ref{lem:bound_on_variance}, we get
    \begin{align}
        R_T &\leq \sqrt{C_1T\beta_T\left(\gamma_n + \frac{1}{2} \sum_{i = n+1}^T \log(1 + \sigma^{-2}_0 \left(1 - C_2 \|\bm u_n\|_2^2\right))\right)} + \frac{\pi^2}{6} \nonumber\\
        &\leq \sqrt{C_1T\beta_T\left(\gamma_n + \frac{\sigma^{-2}_0}{2} (T-n) \left(1 -  C_2 \|\bm u_n\|_2^2\right)\right)} + \frac{\pi^2}{6} \label{eq:proof_use_log_1x_leq_x},
    \end{align}
    where~\eqref{eq:proof_use_log_1x_leq_x} uses $\log(1 + x) \leq x$ for any $x > 0$.
\end{proof}

For the sake of completeness, we conclude this appendix by showing that the information gain $\gamma_n$ scales linearly with $n$ in the time-varying setting.

\begin{lemma}
    For all $n \in \mathbb{N}$, $\gamma_n \in \mathcal{O}(n)$.
\end{lemma}

\begin{proof}
    Recall that $\gamma_n$ can be recovered from the covariance matrix $k(\mathcal{D}, \mathcal{D})$ and its spectrum. More specifically, denoting $\{\lambda_i\}_{1 \leq i \leq n}$ the ordered spectrum of $k(\mathcal{D}, \mathcal{D})$, we have
    \begin{align}
        \gamma_n &= \frac{1}{2} \sum_{i = 1}^n \log(1 + \sigma^{-2}_0 \lambda_i) \nonumber\\
        &\leq \frac{1}{2} \sigma^{-2}_0 \sum_{i = 1}^n \lambda_i \label{eq:proof_logx_x}\\
        &= \frac{1}{2} \sigma^{-2}_0 \Tr(k(\mathcal{D}, \mathcal{D})) \label{eq:proof_trace}\\
        &= \frac{1}{2} \sigma^{-2}_0 n,\label{eq:proof_statio_normalized}
    \end{align}
    where~\eqref{eq:proof_logx_x} uses the property $\log(1 + x) \leq x$ for all $x \in \mathbb{R}_{>0}$, \eqref{eq:proof_trace} uses the property that the trace of a matrix is the sum of its eigenvalues, and~\eqref{eq:proof_statio_normalized} holds for any kernel that is stationary and normalized (i.e., that satisfies $k((\bm x, t), (\bm x, t)) = 1$ for all $(\bm x, t) \in \mathcal{S} \times \mathcal{T}$).

    One immediately concludes from~\eqref{eq:proof_statio_normalized} that $\gamma_n \in \mathcal{O}(n)$, which also concludes the proof.
\end{proof}

\section{Sufficient Conditions for Benefiting from a Maximal Dataset Size} \label{app:sufficient_cond}

In this appendix, we derive sufficient conditions for $\|\bm u_n\|^2_2$ to admit a finite maximizer $n^*$. According to Recommendation~1, these become sufficient conditions for characterizing when capping the dataset size of a TVBO algorithm is beneficial. Throughout, $p$ denotes the polynomial decay exponent of the temporal correlation function $k_T$ and $q$ the polynomial growth exponent of the response time $R$, while $c_T$ and $c_R$ denote the associated multiplicative constants.

\begin{proposition}
\label{prop:finite-n-star}
Suppose that there exist $p, q, c_T, c_R > 0$ such that
\[
|k_T(\tau)| \leq c_T (1 + \tau)^{-p}
\quad\text{and}\quad
R(n) \geq c_R n^q
\]
for every $\tau \geq 0$ and $n \geq 1$. Then
\[
\|\bm u_n\|_2^2
\leq
c\, n^{-2pq} \sum_{i = 1}^n i^{-2p}
\]
for some constant $c > 0$. Consequently,
\[
\|\bm u_n\|_2^2
=
\begin{cases}
\mathcal{O}\left(n^{-2pq}\right), & p > 1/2,\\[1mm]
\mathcal{O}\left(n^{-q}\log n\right), & p = 1/2,\\[1mm]
\mathcal{O}\left(n^{1 - 2p(1+q)}\right), & p < 1/2.
\end{cases}
\]
In particular, if $2p(1+q) > 1$, then $\|\bm u_n\|^2_2 \to 0$ and the sequence admits a finite maximizer.
\end{proposition}

\begin{proof}
Since $\tau \mapsto c_T(1+\tau)^{-p}$ is nonincreasing and $R(n) \geq c_R n^q$, we have, for every $n \geq 1$ and every $i \in [n]$,
\[
k_T^2(iR(n))
\leq
c_T^2 \left(1 + i c_R n^q\right)^{-2p}
\leq
c_T^2 \left(i c_R n^q\right)^{-2p}
=
c\, n^{-2pq} i^{-2p},
\]
where $c = c_T^2 c_R^{-2p}$. Summing this inequality over $i = 1, \ldots, n$ gives the first claim. The three asymptotic regimes then follow from the standard behavior of the partial sums of $i^{-2p}$, which are bounded when $2p > 1$, $\mathcal{O}(\log n)$ when $2p = 1$, and $\mathcal{O}(n^{1-2p})$ when $2p < 1$.

For the last claim, observe that $2p \geq 1$ implies $2p(1+q) \geq 1 + q > 1$, so that the single condition $2p(1+q) > 1$ includes the first two regimes. In those two regimes, $\|\bm u_n\|^2_2$ is $\mathcal{O}(n^{-2pq})$ and $\mathcal{O}(n^{-q}\log n)$ respectively, both of which vanish because $q > 0$. In the third regime, $2p(1+q) > 1$ makes the exponent $1 - 2p(1+q)$ negative, so $\|\bm u_n\|^2_2 \to 0$ as well. A nonnegative sequence converging to $0$ attains its supremum at a finite index, which concludes the proof.
\end{proof}

The condition $2p(1+q) > 1$ admits a simple interpretation: capping the dataset size is beneficial when the temporal correlation decays fast enough (large $p$) relative to how fast the response time grows (large $q$). It fails only when $2p(1+q) \leq 1$, which, since $q > 0$, forces $p < 1/2$. The next proposition shows that, in this regime, $\|\bm u_n\|^2_2$ diverges when the polynomial decay of $k_T$ is tight.

\begin{proposition}
\label{prop:infinite-n-star}
Suppose that there exist $p, q, c'_T, c'_R > 0$ such that
\[
k_T(\tau) \geq c'_T (1 + \tau)^{-p}
\quad\text{and}\quad
R(n) \leq c'_R n^q
\]
for every $\tau \geq 0$ and $n \geq 1$. If $2p(1+q) < 1$, then $\|\bm u_n\|^2_2 \to +\infty$, and the sequence admits no finite maximizer.
\end{proposition}

\begin{proof}
For every $n \geq 1$ and every $i \in [n]$, we have $1 \leq i n^q$, hence $1 + iR(n) \leq 1 + i c'_R n^q \leq (1 + c'_R) i n^q$. Therefore,
\[
k_T^2(iR(n))
\geq
c'^2_T \left(1 + iR(n)\right)^{-2p}
\geq
c'\, n^{-2pq} i^{-2p},
\]
where $c' = c'^2_T (1 + c'_R)^{-2p}$. Since $q > 0$, the assumption $2p(1+q) < 1$ implies $2p < 1$, so that
\[
\sum_{i = 1}^n i^{-2p}
\geq
\int_1^{n+1} x^{-2p} dx
=
\frac{(n+1)^{1-2p} - 1}{1 - 2p}.
\]
Summing the pointwise bound over $i = 1, \ldots, n$ therefore gives $\|\bm u_n\|^2_2 \in \Omega\left(n^{1 - 2p(1+q)}\right)$, which diverges because $1 - 2p(1+q) > 0$. A diverging sequence admits no finite maximizer.
\end{proof}

\begin{remark}
\label{rem:uncovered-regime}
Propositions~\ref{prop:finite-n-star} and~\ref{prop:infinite-n-star} leave out the edge case $2p(1+q) = 1$, at which the upper bound of the former is $\mathcal{O}(1)$ and the lower bound of the latter is $\Omega(1)$. At this order the analysis is inconclusive, and whether $n^*$ is finite or not has to be determined by another method.

The remaining regime $2p(1+q) < 1$ is, however, unreachable under Assumption~\ref{ass:kernel_props}. Indeed, suppose that the envelope satisfies $g_T(\tau) = \Theta(\tau^{-p})$ and that the response time satisfies $R(m) = \Theta(m^q)$ with $q > 0$. Then the finiteness of $G$ required by Assumption~\ref{ass:kernel_props} already imposes $p(1+q) \geq 1$, hence $2p(1+q) \geq 2$. Assumption~\ref{ass:kernel_props} is therefore more demanding than the condition $2p(1+q) > 1$ of Proposition~\ref{prop:finite-n-star}, so that the divergence regime of Proposition~\ref{prop:infinite-n-star} and the edge case $2p(1+q) = 1$ both fall outside the setting of this paper.

Consequently, under Assumptions~\ref{ass:gp}-\ref{ass:lipschitz} and a response time growing polynomially with $q > 0$, every temporal correlation function with a polynomially decaying envelope admits a finite $n^*$: capping the dataset size is always beneficial in this setting. This does not, however, confine the analysis to the first regime of Proposition~\ref{prop:finite-n-star}. Because Assumption~\ref{ass:kernel_props} constrains the product $p(1+q)$ rather than the decay exponent $p$ alone, all three regimes remain reachable: an envelope with $p = 1/5$ combined with a response time with $q = 4$, for instance, satisfies $p(1+q) = 1$, falls in the third regime, and yields $\|\bm u_n\|^2_2 \in \mathcal{O}(n^{-1})$.
\end{remark}

\paragraph{Two borderline examples: sinc and RQ-1/2.}
Let us instantiate the above discussion on the two temporal correlation functions of Figure~\ref{fig:max_dataset_size} that sit exactly at the boundary of validity of Assumption~\ref{ass:kernel_props}. The sinc correlation function $k_T(\tau) = \sin(\tau / l_T)/(\tau / l_T)$ satisfies $|k_T(\tau)| \leq \min(1, l_T/\tau)$, while the Rational Quadratic correlation function $k_T(\tau) = \left(1 + \tau^2/(2\alpha l_T^2)\right)^{-\alpha}$ decays as $\Theta(\tau^{-2\alpha})$. Both therefore admit a polynomial envelope with exponent $p = 1$, the former exactly and the latter for $\alpha = 1/2$. Neither is integrable, so neither is dominated by an integrable envelope $g_T$. This is precisely why Assumption~\ref{ass:kernel_props} is stated through the quantity $G$ rather than through $\int_0^\infty g_T(\tau) d\tau$.

With $p = 1$, Assumption~\ref{ass:kernel_props} holds as soon as $p(1+q) = 1 + q > 1$, that is, as soon as $q > 0$, and Proposition~\ref{prop:finite-n-star} applies because $2p(1+q) = 2(1+q) > 1$. The resulting rate is $\|\bm u_n\|^2_2 \in \mathcal{O}(n^{-2q})$, so both correlation functions admit a finite $n^*$ as soon as the response time grows polynomially with the dataset size. For the response time $R(n) \in \Theta(n^3)$ used in the left panel of Figure~\ref{fig:max_dataset_size}, this gives $\|\bm u_n\|^2_2 \in \mathcal{O}(n^{-6})$, which is consistent with the finite maxima plotted there. More generally, for a RQ correlation function with parameter $\alpha$, Assumption~\ref{ass:kernel_props} requires $\alpha > 1/(2(1+q))$ while Proposition~\ref{prop:finite-n-star} requires $\alpha > 1/(4(1+q))$. For $q = 3$, these thresholds are $\alpha > 1/8$ and $\alpha > 1/16$ respectively, so that $\alpha = 1/2$ lies comfortably inside both.

The only case in which these two correlation functions fall outside Assumption~\ref{ass:kernel_props} is $q = 0$, that is, a response time that does not grow with the dataset size. Then $\int_0^{mR(m)} g_T(\tau) d\tau$ grows logarithmically in $m$ while $R(m)$ remains constant, so that $G$ diverges. However, in that setting, $\|\bm u_n\|^2_2$ is nondecreasing and Recommendation~1 returns $n^* = T$ irrespective of the temporal correlation function. Capping the dataset size is not beneficial in that regime, so nothing is lost by excluding it.

\section{Recommended Stale Data Management Policy} \label{app:recommended_stale}

In this appendix, we prove Theorem~\ref{thm:acquisition_wasserstein} by connecting the $L^2$-distance between Lipschitz-continuous acquisition functions and the integrated 2-Wasserstein distance between GP posteriors.

\begin{proof}
    Let us start by considering the effect of Assumption~\ref{ass:acquisition_lipschitz} on the distance at a single point $(\bm x, t) \in \mathcal{S} \times \mathcal{T}$.
    \begin{align}
        |\alpha_{\mathcal{D}}(\bm x, t) - \alpha_{\Tilde{\mathcal{D}}}(\bm x, t)| &\leq L_T ||\left(\mu_{\mathcal{D}}(\bm x, t) - \mu_{\Tilde{\mathcal{D}}}(\bm x, t), \sigma_{\mathcal{D}}(\bm x, t) - \sigma_{\Tilde{\mathcal{D}}}(\bm x, t)\right)||_2 \nonumber\\
        &= L_T \sqrt{\left(\mu_{\mathcal{D}}(\bm x, t) - \mu_{\Tilde{\mathcal{D}}}(\bm x, t)\right)^2 + \left(\sigma_{\mathcal{D}}(\bm x, t) - \sigma_{\Tilde{\mathcal{D}}}(\bm x, t)\right)^2}. \label{eq:proof_acq_single_point_distance}
    \end{align}

    We can now bound the squared $L^2$ distance between the two acquisition functions $\alpha_{\mathcal{D}}$ and $\alpha_{\Tilde{\mathcal{D}}}$ on any subset $\mathcal{S}' \times \mathcal{T}'$ of the spatio-temporal domain $\mathcal{S} \times \mathcal{T}$:
    \begin{align}
        ||\alpha_{\mathcal{D}} - \alpha_{\Tilde{\mathcal{D}}}||^2_2 &= \oint_{\mathcal{S}'} \int_{\mathcal{T}'} \left(\alpha_{\mathcal{D}}(\bm x, t) - \alpha_{\Tilde{\mathcal{D}}}(\bm x, t)\right)^2d\bm x dt \nonumber\\
        &\leq \oint_{\mathcal{S}'} \int_{\mathcal{T}'} L_T^2 \left(\left(\mu_{\mathcal{D}}(\bm x, t) - \mu_{\Tilde{\mathcal{D}}}(\bm x, t)\right)^2 + \left(\sigma_{\mathcal{D}}(\bm x, t) - \sigma_{\Tilde{\mathcal{D}}}(\bm x, t)\right)^2\right) d\bm x dt \label{eq:proof_application_acq_single_point}\\
        &= L_T^2 \oint_{\mathcal{S}'} \int_{\mathcal{T}'} \left(\left(\mu_{\mathcal{D}}(\bm x, t) - \mu_{\Tilde{\mathcal{D}}}(\bm x, t)\right)^2 + \left(\sigma_{\mathcal{D}}(\bm x, t) - \sigma_{\Tilde{\mathcal{D}}}(\bm x, t)\right)^2\right) d\bm x dt \nonumber\\
        &= L_T^2 W_2^2(\mathcal{GP}_{\mathcal{D}}, \mathcal{GP}_{\Tilde{\mathcal{D}}}), \label{eq:proof_distance_acq}
    \end{align}
    where~\eqref{eq:proof_application_acq_single_point} follows from~\eqref{eq:proof_acq_single_point_distance} and~\eqref{eq:proof_distance_acq} from the definition of the integrated 2-Wasserstein distance on the domain $\mathcal{S}' \times \mathcal{T}'$. Applying a square-root on both sides yields the desired result.
\end{proof}

\section{Numerical Results} \label{app:benchmarks}

Here, we provide a detailed description of each implemented benchmark and the associated figures. There are two figures associated with each benchmark, showing their average regrets and the size of their datasets throughout the experiment.

In the following, the synthetic benchmarks will be described as functions of a point $\bm z$ in the ($d+1$)-dimensional spatio-temporal domain $\mathcal{S} \times \mathcal{T}$. More precisely, the point $\bm z$ is explicitly given by $\bm z = (x_1, \cdots, x_d, t)$. Also, we will write $d' = d+1$ for the sake of brevity.

Each benchmark comes with a time horizon $H$. To make the benchmarks noisy, each call to the objective function is perturbed with a centered Gaussian noise of variance $\sigma^2_0$, equal to $1\%$ of the signal variance. Moreover, to control the expensiveness of the benchmarks, each call to the objective function takes $R(0)$ seconds to complete. The values of $H$, $\sigma^2_0$ and $R(0)$ are unknown to the evaluated TVBO algorithms but are provided in Table~\ref{tab:parameters} for the sake of completeness.

\begin{table}[t]
\caption{Noise variance $\sigma^2_0$, cost of single call $R(0)$ in seconds, time horizon $H$ in seconds and approximated temporal lengthscale $l_T$ in minutes for each benchmark. The table also indicates if each benchmark can be viewed as quasi-static ($l_T \gg R(0)$) or if it can be viewed as near-constant ($R(0) \sim R(\lfloor H/ R(0)\rfloor)$).}
\label{tab:parameters}
\vskip 0.15in
\begin{center}
\begin{tabular}{ccccccc}
\toprule
Benchmark & Observ. & Cost & Horizon & Lengthscale & Quasi- & Near-\\
$(d+1)$ & Noise $\sigma^2_0$ & $R(0)$ (s) & $H$ (s) & $l_T$ (s) & Static & Constant\\

\midrule
Shekel (4) & 0.02 & 8.00 & 600 & 174 & No & Yes\\
Hartmann (3) & 0.05 & 8.00 & 600 & 164 & No & Yes\\
Ackley (4) & 0.05 & 0.05 & 600 & 216 & Yes & No\\
Griewank (6) & 0.30 & 0.05 & 600 & 55 & No & No\\
Eggholder (2) & 0.10 & 0.05 & 600 & 22 & No & No\\
Schwefel (4) & 0.25 & 0.05 & 600 & 35 & No & No\\
Hartmann (6) & 0.05 & 0.10 & 600 & 270 & Yes & No\\
Powell (4) & 2.50 & 0.01 & 600 & 552 & Yes & No\\
\midrule
Temperature (3) & 0.16 & 0.01 & 1800 & 396 & Yes & No\\
WLAN (7) & 1.50 & 0.10 & 600 & 120 & No & No\\
\bottomrule
\end{tabular}
\end{center}
\vskip -0.1in
\end{table}

The average regrets and their standard errors are reported for each method and each benchmark in Table~\ref{tab:eval}.

\begin{table*}[t]
\caption{Comparison of \AlgoName\ against state-of-the-art TVBO algorithms. For each experiment and each algorithm, the average regret and its standard error over 10 independent experiments is provided (lower is better). For each experiment, the performance of the best algorithm is written \textbf{in bold}, and the performance of algorithms whose confidence intervals overlap the confidence interval of the best performing algorithm are \underline{underlined}. Due to space constraints, each instance of "GPUCB" in the methods' names has been replaced by "G", e.g., "TV-G" is "TV-GPUCB".}
\label{tab:eval}
\small
\vskip 0.15in
\begin{sc}
\begin{center}
\begin{tabular}{ccccccccc}
\toprule
Bench. ($d+1$) & G & ET-G & R-G & TV-G & ABO & W-DBO & FIFO & \AlgoName\\
\midrule
\multirow{2}*{Shekel (4)} & \underline{2.57} & \underline{2.57} & 2.58 & 2.62 & 2.67 & 2.64 & \underline{2.55} & \textbf{2.51}\\
& \underline{$\pm$0.05} & \underline{$\pm$0.04} & $\pm$0.02 & $\pm$0.05 & $\pm$0.11 & $\pm$0.04 & \underline{$\pm$0.04} & $\mathbf{\pm}$\textbf{0.02}\\
\midrule
\multirow{2}*{Hartmann (3)} & 1.58 & 1.43 & 0.70 & 0.77 & 0.91 & 0.63 & \textbf{0.51} & \underline{0.60} \\
 & $\pm$0.08 & $\pm$0.12 & $\pm$0.05 & $\pm$0.10 & $\pm$0.12 & $\pm$0.04 & \textbf{$\pm$0.06} & \underline{$\pm$0.04} \\
\midrule
\multirow{2}*{Ackley (4)} & 4.26 & 4.39 & 4.28 & 3.83 & 4.75 & \underline{3.42} & 3.74 & \textbf{2.92} \\
 & $\pm$0.50 & $\pm$0.45 & $\pm$0.20 & $\pm$0.57 & $\pm$0.53 & \underline{$\pm$0.46} & $\pm$0.45 & $\mathbf{\pm}$\textbf{0.35} \\
\midrule
\multirow{2}*{Griewank (6)} & \underline{0.59} & 0.61 & 0.61 & 0.66 & 0.69 & \textbf{0.57} & \underline{0.62} & \underline{0.58} \\
& \underline{$\pm$0.01} & $\pm$0.01 & $\pm$0.01 & $\pm$0.02 & $\pm$0.03 & $\mathbf{\pm}$\textbf{0.03} & \underline{$\pm$0.03} & \underline{$\pm$0.02} \\
\midrule
\multirow{2}*{Eggholder (2)} & 517 & 522 & 298 & 305 & 546 & 277 & \textbf{225} & 262 \\
 & $\pm$16.4 & $\pm$13.2 & $\pm$4.3 & $\pm$13.3 & $\pm$109.9 & $\pm$12.4 & $\mathbf{\pm}$\textbf{6.8} & $\pm$7.1 \\
\midrule
\multirow{2}*{Schwefel (4)} & 590 & 589 & 1024 & 726 & 792 & 615 & 735 & \textbf{523} \\
 & $\pm$20.9 & $\pm$35.3 & $\pm$10.3 & $\pm$34.1 & $\pm$61.3 & $\pm$21.5 & $\pm$43.0 & $\mathbf{\pm}$\textbf{30.3} \\
\midrule
\multirow{2}*{Hartmann (6)} & 1.67 & 1.75 & 1.44 & \underline{0.96} & 1.32 & \textbf{0.70} & 1.20 & \underline{0.72} \\
& $\pm$0.01 & $\pm$0.07 & $\pm$0.02 & \underline{$\pm$0.24} & $\pm$0.17 & $\mathbf{\pm}$\textbf{0.03} & $\pm$0.14 & \underline{$\pm$0.06} \\
\midrule
\multirow{2}*{Powell (4)} & 3429 & 3283 & 1422 & 1301 & 6721 & 1066 & 1070 & \textbf{547} \\
 & $\pm$153 & $\pm$240 & $\pm$64 & $\pm$155 & $\pm$2113 & $\pm$149 & $\pm$202 & $\mathbf{\pm}$\textbf{166} \\
\midrule
\multirow{2}*{Temperature (3)} & 1.12 & 1.29 & \textbf{0.70} & 1.25 & 1.41 & 0.98 & 0.83 & \underline{0.72} \\
& $\pm$0.05 & $\pm$0.11 & $\mathbf{\pm}$\textbf{0.01} & $\pm$0.10 & $\pm$0.26 & $\pm$0.06 & $\pm$0.09 & \underline{$\pm$0.03} \\
\midrule
\multirow{2}*{WLAN (7)} & 20.2 & 21.2 & 10.4 & \underline{7.8} & 17.3 & 8.9 & \underline{7.9} & \textbf{7.7} \\
 & $\pm$1.58 & $\pm$1.05 & $\pm$0.16 & \underline{$\pm$0.34} & $\pm$1.93 & $\pm$0.20 & \underline{$\pm$0.17} & $\mathbf{\pm}$\textbf{0.15} \\
\midrule
\multirow{2}*{\textbf{Overall}} & 0.62 & 0.67 & 0.39 & 0.40 & 0.82 & 0.21 & 0.23 & \textbf{0.04} \\
 & $\pm$0.10 & $\pm$0.10 & $\pm$0.10 & $\pm$0.08 & $\pm$0.07 & $\pm$0.07 & $\pm$0.06 & $\mathbf{\pm}$\textbf{0.02} \\
\bottomrule
\end{tabular}
\end{center}
\vskip -0.1in
\end{sc}
\end{table*}

\begin{figure}[t]
    \centering
    \includegraphics[height=5cm]{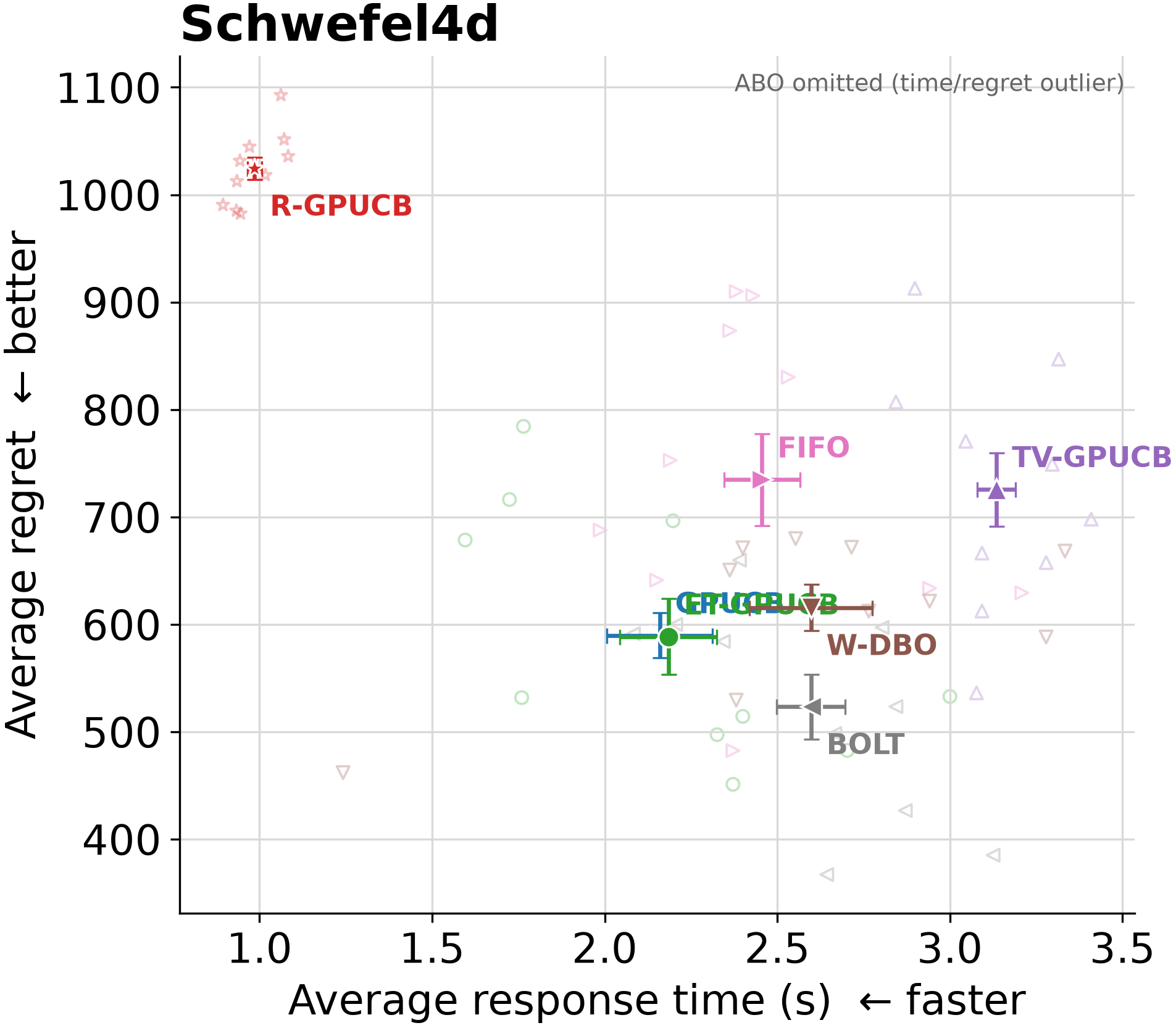} \quad
    \includegraphics[height=5cm]{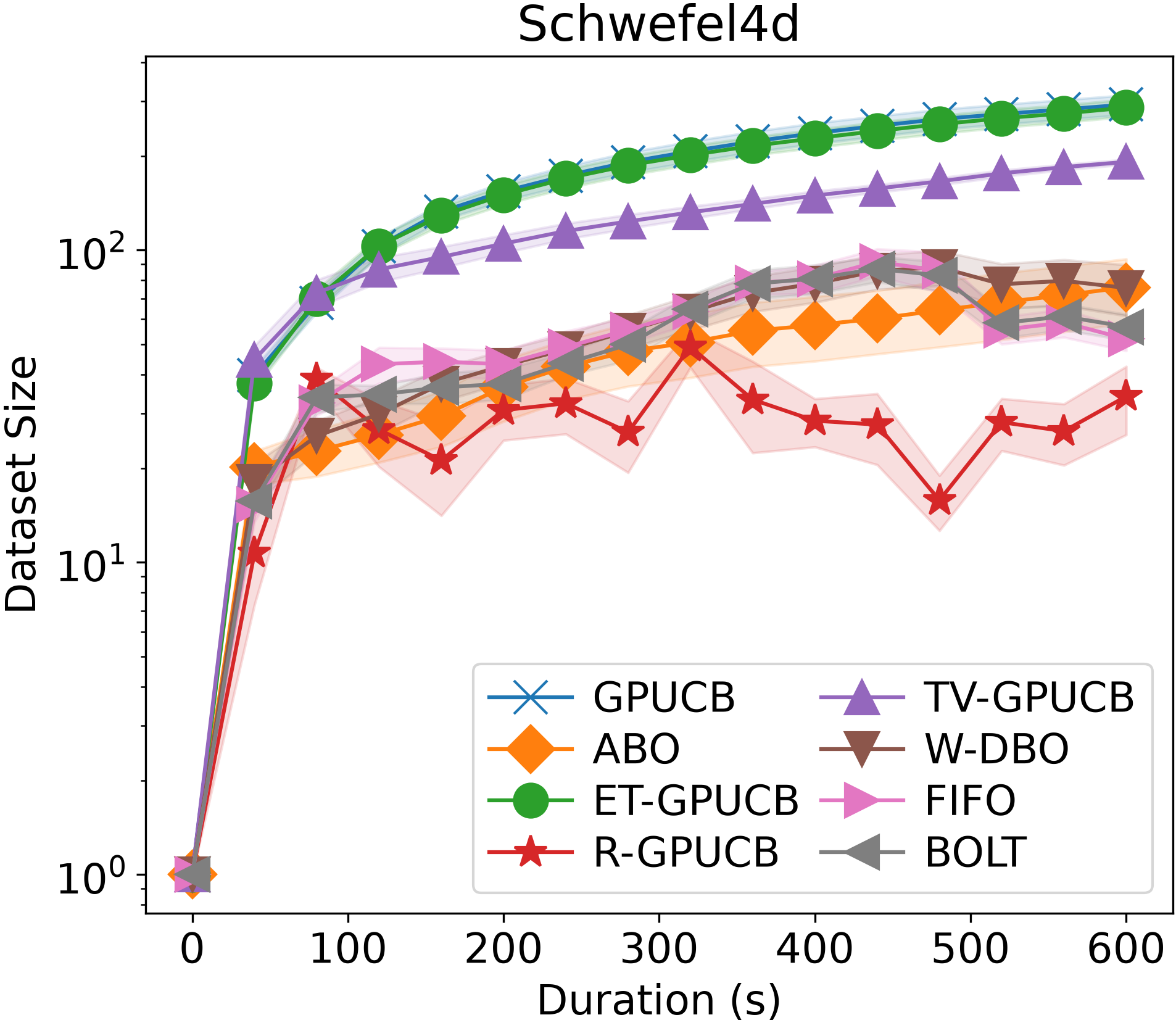}
    \caption{(Left) Average response time and average regrets of the TVBO solutions during the optimization of the Schwefel synthetic function. (Right)~Dataset sizes of the TVBO solutions during the optimization of the Schwefel synthetic function.}
    \label{fig:schwefel}
\end{figure}

\paragraph{Schwefel.} The Schwefel function is $d'$-dimensional, and has the form
\begin{equation*}
    f(\bm z) = 418.9829d' - \sum_{i = 1}^{d'} z_i \sin\left(\sqrt{|z_i|}\right).
\end{equation*}

For the numerical evaluation, we set $d' = 4$ and we optimized the function on the domain $[-500, 500]^{d'}$. The results are provided in Figure~\ref{fig:schwefel}.

\begin{figure}[t]
    \centering
    \includegraphics[height=5.7cm]{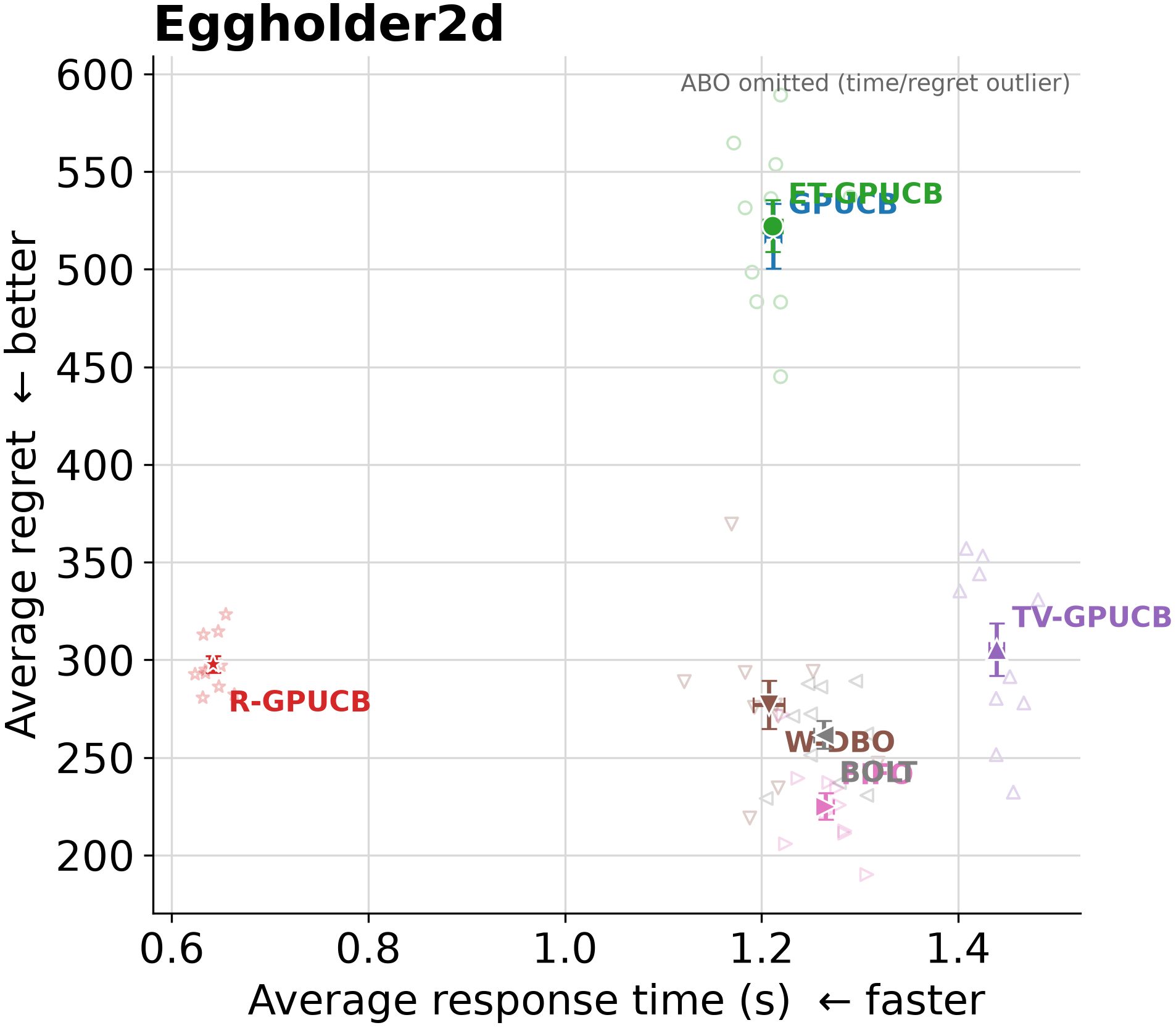} \quad
    \includegraphics[height=5.7cm]{rsc/Eggholder2d_Dataset.png}
    \caption{(Left) Average response time and average regrets of the TVBO solutions during the optimization of the Eggholder synthetic function. (Right)~Dataset sizes of the TVBO solutions during the optimization of the Eggholder synthetic function.}
    \label{fig:eggholder}
\end{figure}

\paragraph{Eggholder.} The Eggholder function is $2$-dimensional, and has the form
\begin{equation*}
    f(\bm z) = -(z_2 + 47)\sin\left(\sqrt{|z_2 + \frac{z_1}{2} + 47|}\right) - z_1 \sin\left(\sqrt{|z_1 - z_2 - 47|}\right).
\end{equation*}

For the numerical evaluation, we optimized the function on the domain $[-512, 512]^2$. The results are provided in Figure~\ref{fig:eggholder}.

\begin{figure}[t]
    \centering
    \includegraphics[height=5.7cm]{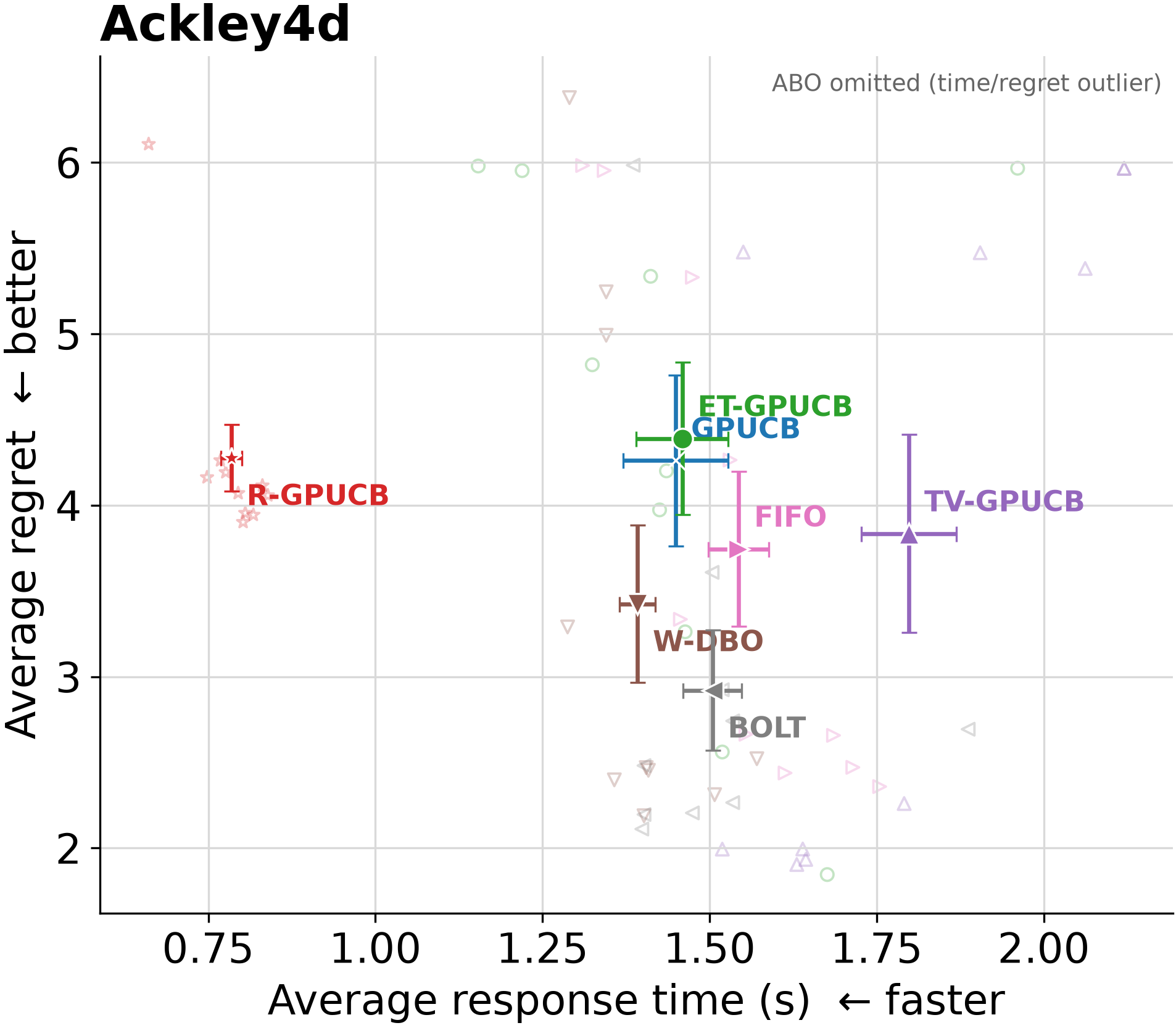} \quad
    \includegraphics[height=5.7cm]{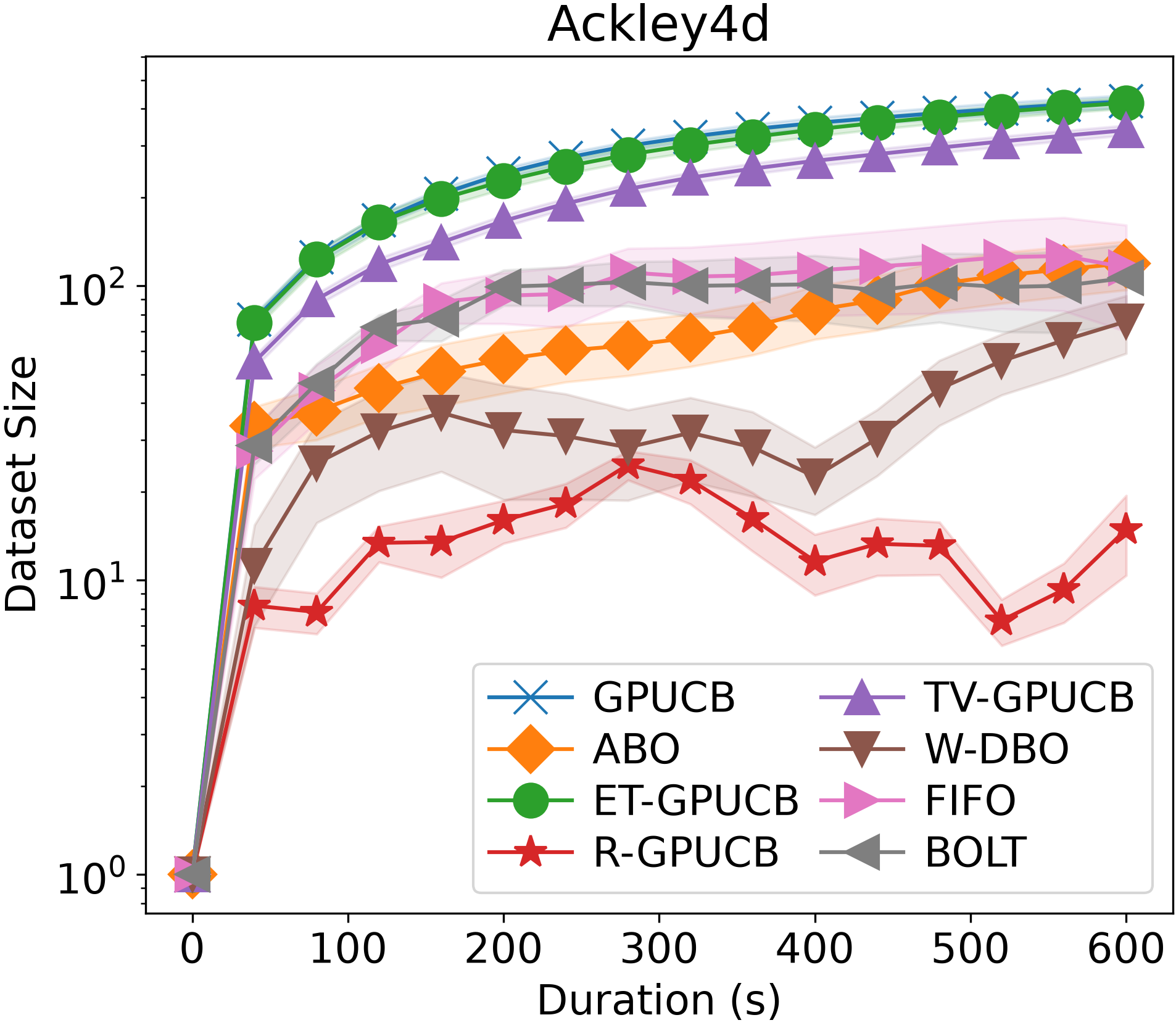}
    \caption{(Left) Average response time and average regrets of the TVBO solutions during the optimization of the Ackley synthetic function. (Right)~Dataset sizes of the TVBO solutions during the optimization of the Ackley synthetic function.}
    \label{fig:ackley}
\end{figure}

\paragraph{Ackley.} The Ackley function is $d'$-dimensional, and has the form
\begin{equation*}
    f(\bm z) = -a\exp\left(-b \sqrt{\frac{1}{d'} \sum_{i = 1}^{d'} z_i^2}\right) - \exp\left(\frac{1}{d'} \sum_{i = 1}^{d'} \cos(cz_i)\right) + a + \exp(1).
\end{equation*}

For the numerical evaluation, we set $a = 20$, $b = 0.2$, $c=2\pi$, $d' = 4$ and we optimized the function on the domain $[-32, 32]^{d'}$. The results are provided in Figure~\ref{fig:ackley}.

\begin{figure}[t]
    \centering
    \includegraphics[height=5.7cm]{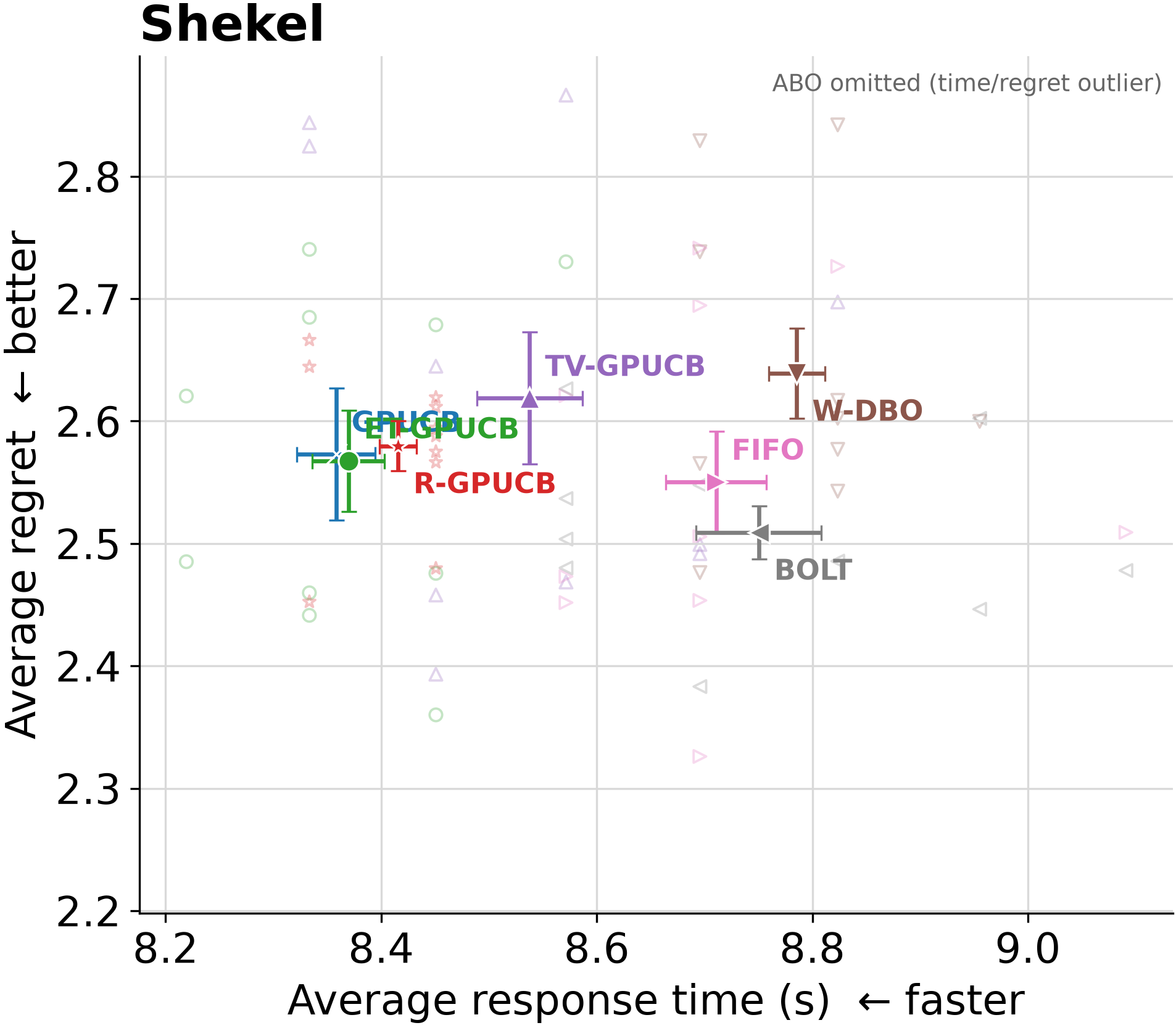} \quad
    \includegraphics[height=5.7cm]{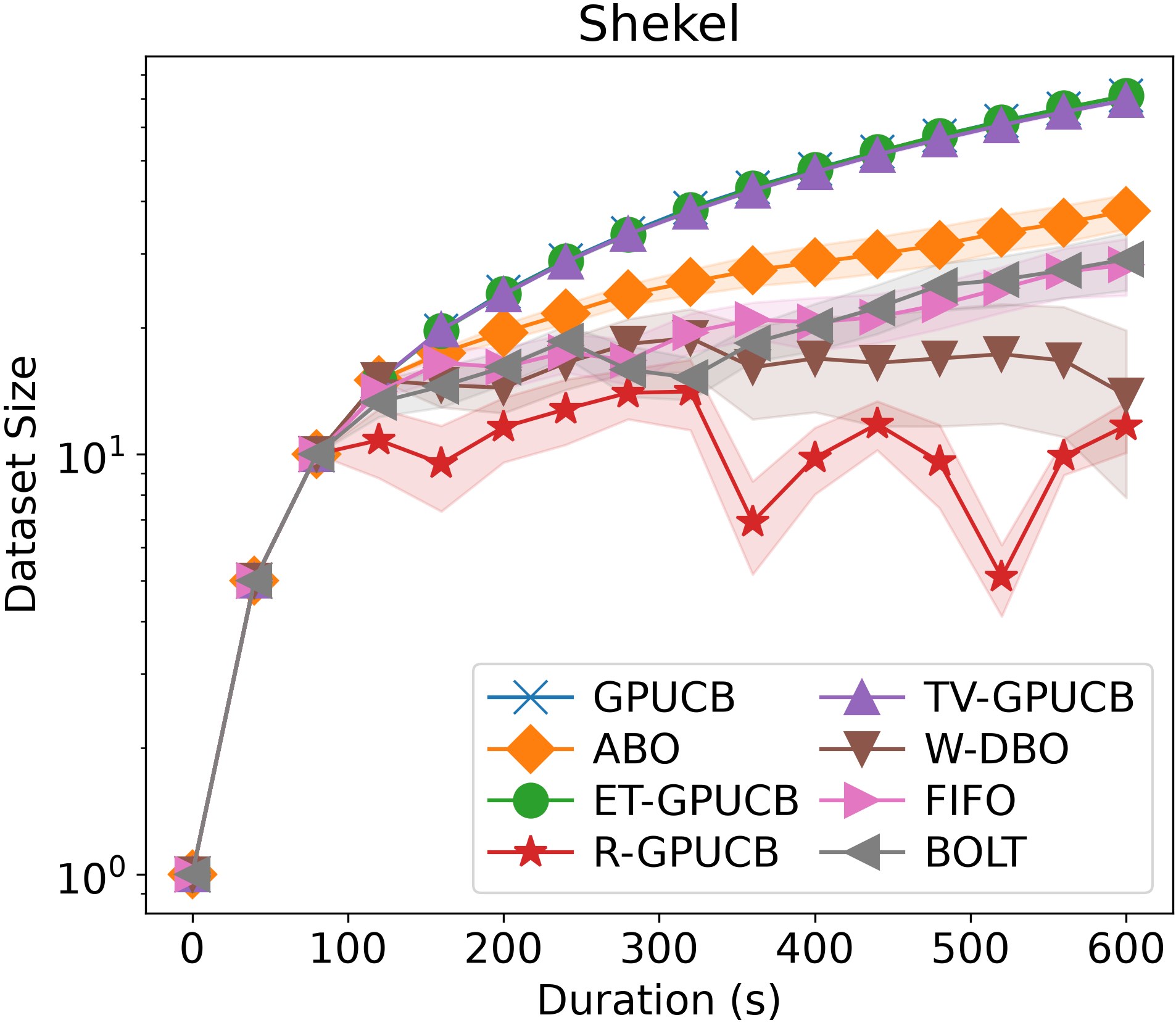}
    \caption{(Left) Average response time and average regrets of the TVBO solutions during the optimization of the Shekel synthetic function. (Right)~Dataset sizes of the TVBO solutions during the optimization of the Shekel synthetic function.}
    \label{fig:shekel}
\end{figure}

\paragraph{Shekel.} The Shekel function is $4$-dimensional, and has the form
\begin{equation*}
    f(\bm z) = - \sum_{i = 1}^m \left(\sum_{j = 1}^4 (z_j - C_{ji})^2 + \beta_i\right)^{-1}.
\end{equation*}

For the numerical evaluation, we set $m = 10$, $\bm \beta = \frac{1}{10}\left(1, 2, 2, 4, 4, 6, 3, 7, 5, 5\right)$,
\begin{equation*}
    \bm C = \begin{pmatrix}
        4 & 1 & 8 & 6 & 3 & 2 & 5 & 8 & 6 & 7\\
        4 & 1 & 8 & 6 & 7 & 9 & 3 & 1 & 2 & 3.6\\
        4 & 1 & 8 & 6 & 3 & 2 & 5 & 8 & 6 & 7\\
        4 & 1 & 8 & 6 & 7 & 9 & 3 & 1 & 2 & 3.6
    \end{pmatrix},
\end{equation*}
and we optimized the function on the domain $[0, 10]^4$. The results are provided in Figure~\ref{fig:shekel}.

\begin{figure}[t]
    \centering
    \includegraphics[height=5.7cm]{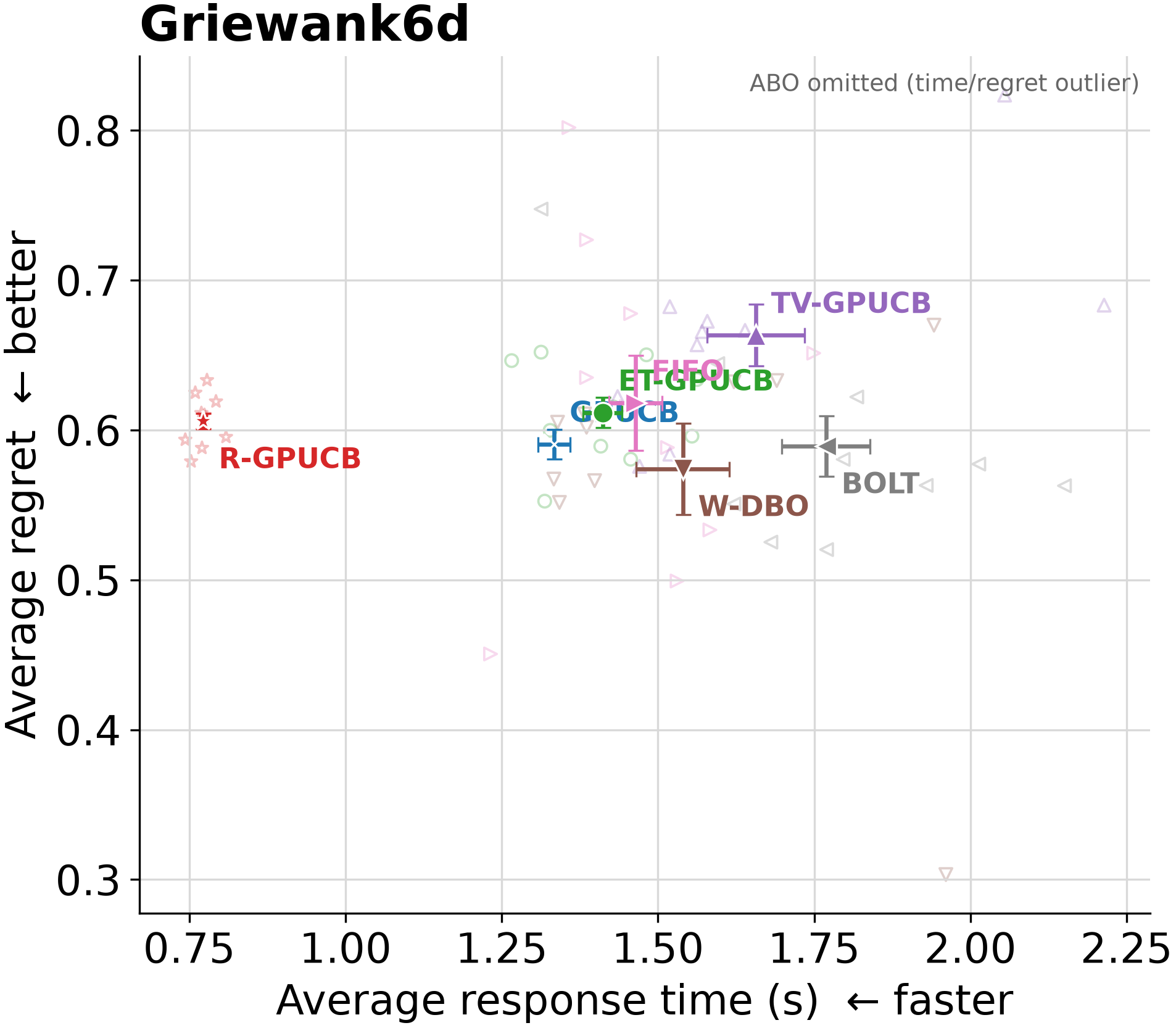} \quad
    \includegraphics[height=5.7cm]{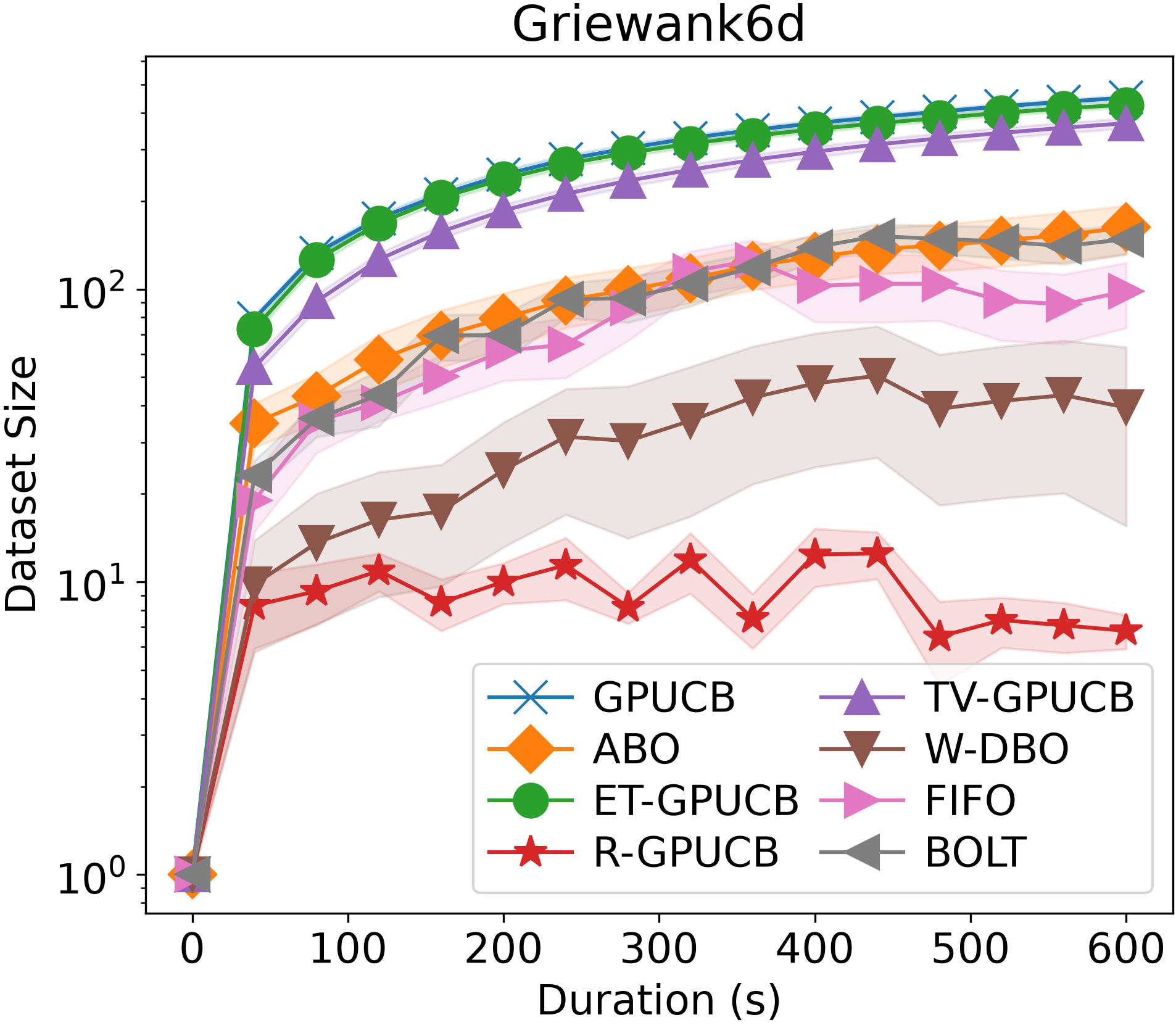}
    \caption{(Left) Average response time and average regrets of the TVBO solutions during the optimization of the Griewank synthetic function. (Right)~Dataset sizes of the TVBO solutions during the optimization of the Griewank synthetic function.}
    \label{fig:griewank}
\end{figure}

\paragraph{Griewank.} The Griewank function is $d'$-dimensional, and has the form
\begin{equation*}
    f(\bm z) = \sum_{i = 1}^{d'} \frac{x_i^2}{4000} - \prod_{i = 1}^{d'} \cos\left(\frac{x_i}{\sqrt{i}}\right) + 1
\end{equation*}

For the numerical evaluation, we set $d' = 6$ and we optimized the function on the domain $[-600, 600]^{d'}$. The results are provided in Figure~\ref{fig:griewank}.

\begin{figure}[t]
    \centering
    \includegraphics[height=5.7cm]{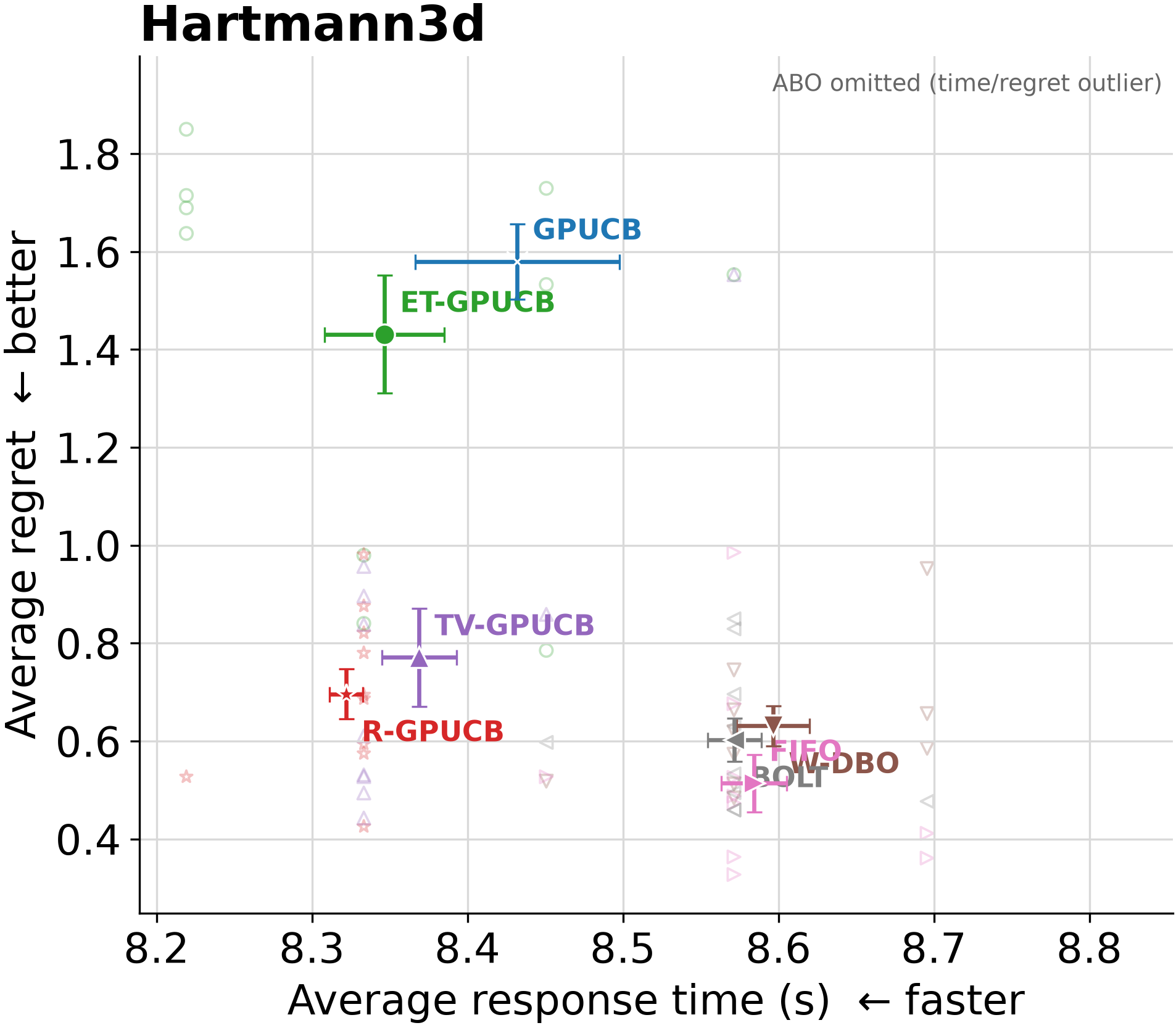} \quad
    \includegraphics[height=5.7cm]{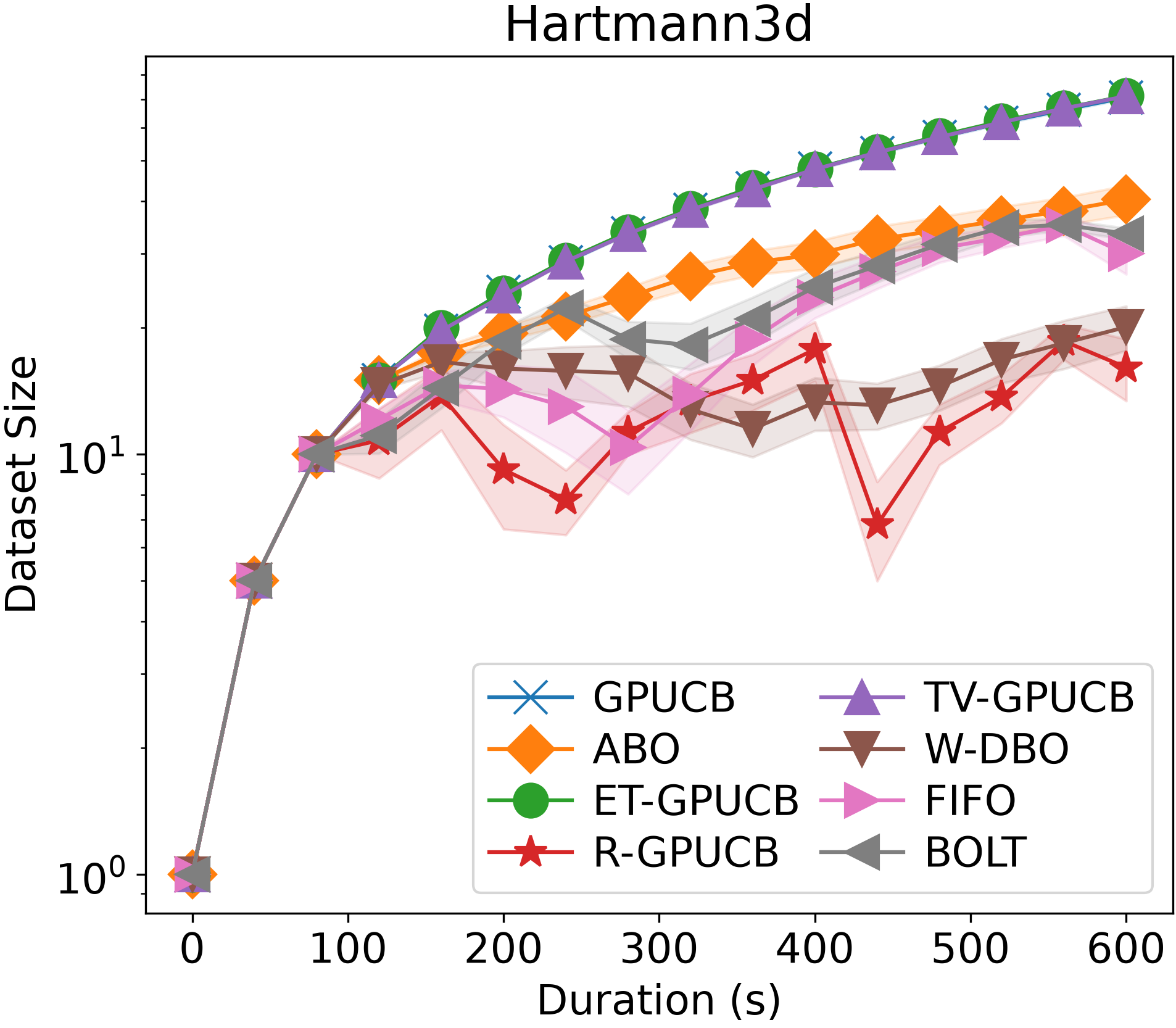}
    \caption{(Left) Average response time and average regrets of the TVBO solutions during the optimization of the Hartmann-3 synthetic function. (Right)~Dataset sizes of the TVBO solutions during the optimization of the Hartmann-3 synthetic function.}
    \label{fig:hartmann3}
\end{figure}

\paragraph{Hartmann-3.} The Hartmann-3 function is $3$-dimensional, and has the form
\begin{equation*}
    f(\bm z) = - \sum_{i = 1}^4 \alpha_i \exp\left(-\sum_{j = 1}^3 A_{ij} (z_j - P_{ij})^2\right).
\end{equation*}

For the numerical evaluation, we set $\bm \alpha = \left(1.0, 1.2, 3.0, 3.2\right)$,
\begin{equation*}
    \bm A = \begin{pmatrix}
        3 & 10 & 30\\
        0.1 & 10 & 35\\
        3 & 10 & 30\\
        0.1 & 10 & 35\\
    \end{pmatrix}, \bm P = 10^{-4}\begin{pmatrix}
        3689 & 1170 & 2673\\
        4699 & 4387 & 7470\\
        1091 & 8732 & 5547\\
        381 & 5743 & 8828
    \end{pmatrix},
\end{equation*}
and we optimized the function on the domain $[0, 1]^3$. The results are provided in Figure~\ref{fig:hartmann3}.

\begin{figure}[t]
    \centering
    \includegraphics[height=5.7cm]{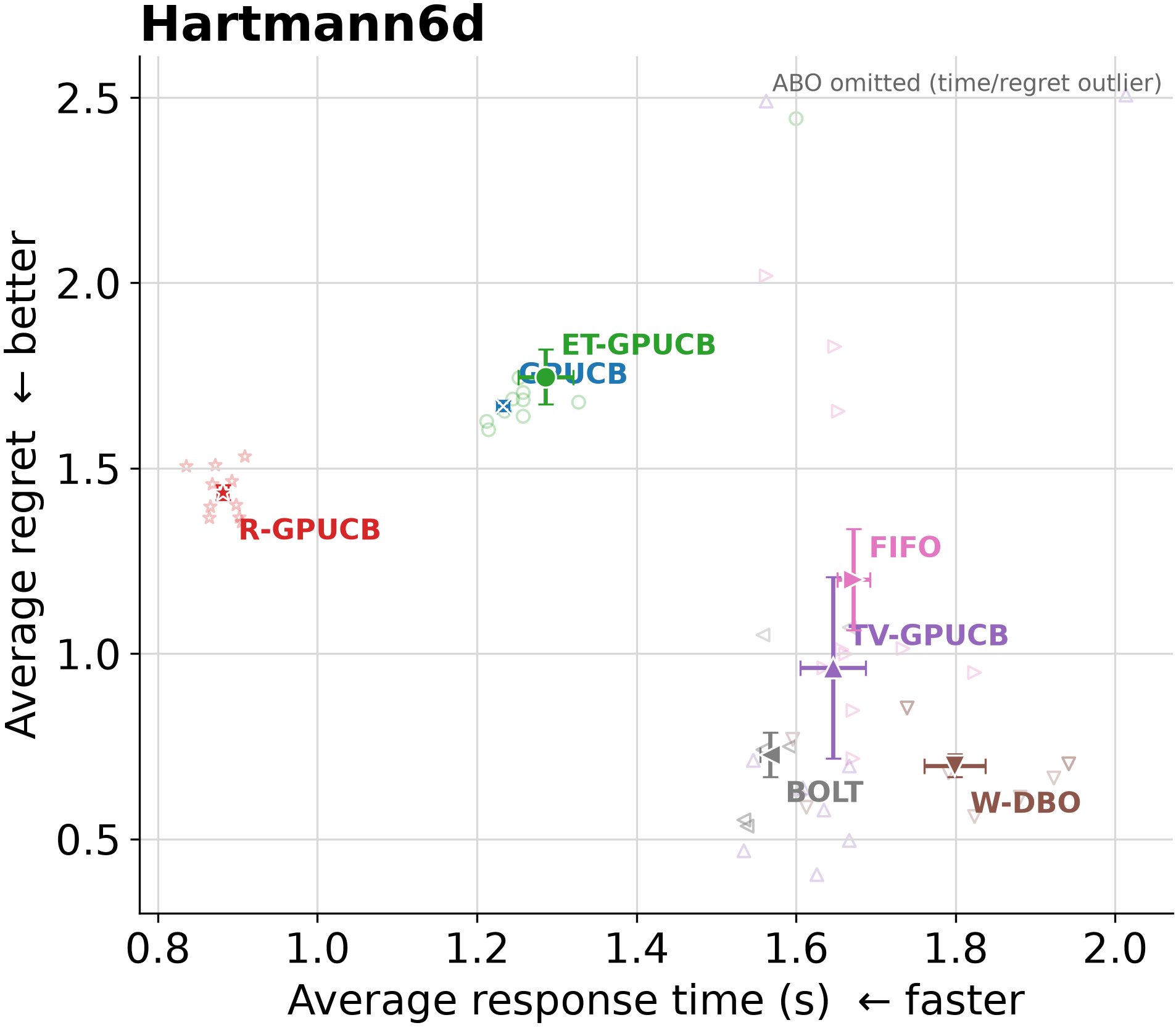} \quad
    \includegraphics[height=5.7cm]{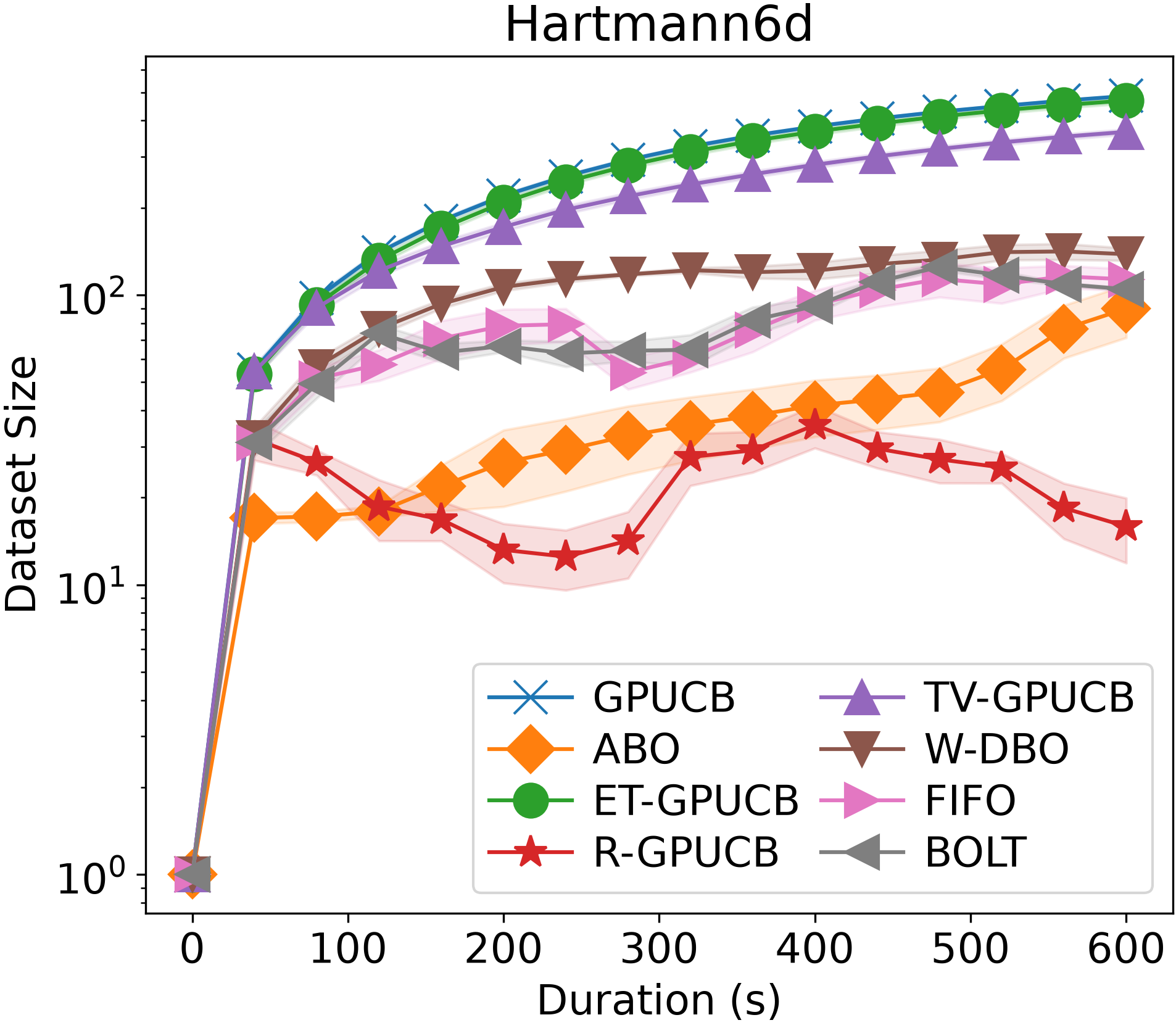}
    \caption{(Left) Average response time and average regrets of the TVBO solutions during the optimization of the Hartmann-6 synthetic function. (Right)~Dataset sizes of the TVBO solutions during the optimization of the Hartmann-6 synthetic function.}
    \label{fig:hartmann6}
\end{figure}

\paragraph{Hartmann-6.} The Hartmann-6 function is $6$-dimensional, and has the form
\begin{equation*}
    f(\bm z) = - \sum_{i = 1}^4 \alpha_i \exp\left(-\sum_{j = 1}^6 A_{ij} (z_j - P_{ij})^2\right).
\end{equation*}

For the numerical evaluation, we set $\bm \alpha = \left(1.0, 1.2, 3.0, 3.2\right)$,
\begin{equation*}
    \bm A = \begin{pmatrix}
        10 & 3 & 17 & 3.50 & 1.7 & 8\\
        0.05 & 10 & 17 & 0.1 & 8 & 14\\
        3 & 3.5 & 1.7 & 10 & 17 & 8\\
        17 & 8 & 0.05 & 10 & 0.1 & 14\\
    \end{pmatrix}, \bm P = 10^{-4}\begin{pmatrix}
        1312 & 1696 & 5569 & 124 & 8283 & 5886\\
        2329 & 4135 & 8307 & 3736 & 1004 & 9991\\
        2348 & 1451 & 3522 & 2883 & 3047 & 6650\\
        4047 & 8828 & 8732 & 5743 & 1091 & 381
    \end{pmatrix},
\end{equation*}
and we optimized the function on the domain $[0, 1]^6$. The results are provided in Figure~\ref{fig:hartmann6}.

\begin{figure}[t]
    \centering
    \includegraphics[height=5.7cm]{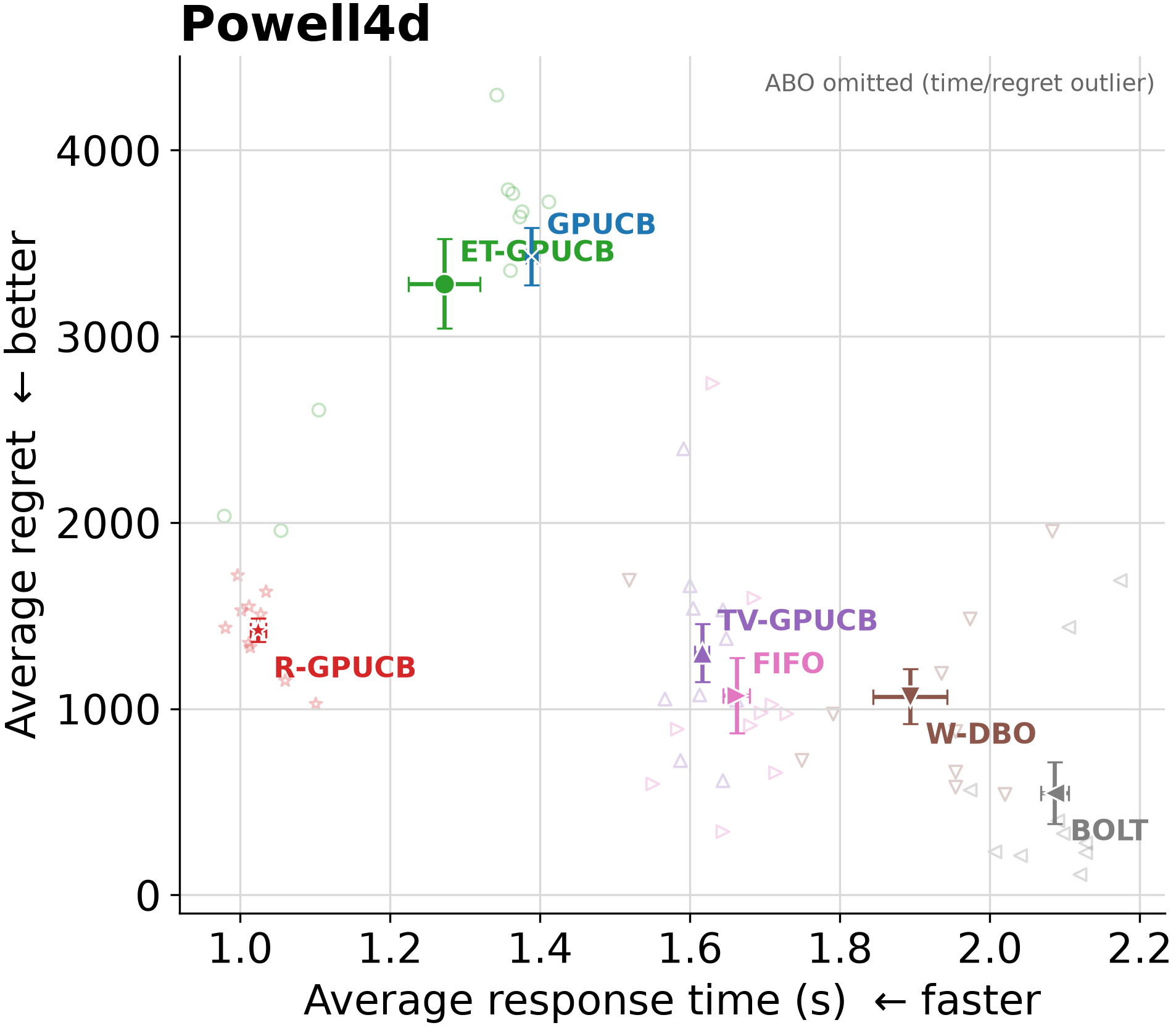} \quad
    \includegraphics[height=5.7cm]{rsc/Powell4d_Dataset.png}
    \caption{(Left) Average response time and average regrets of the TVBO solutions during the optimization of the Powell synthetic function. (Right)~Dataset sizes of the TVBO solutions during the optimization of the Powell synthetic function.}
    \label{fig:powell}
\end{figure}

\paragraph{Powell.} The Powell function is $d'$-dimensional, and has the form
\begin{equation*}
    f(\bm z) = \sum_{i = 1}^{d' / 4} (z_{4i-3} + 10z_{4i-2})^2 + 5(z_{4i - 1} - z_{4i})^2 + (z_{4i-2} - 2z_{4i-1})^4 + 10(z_{4i-3} - z_{4i})^4.
\end{equation*}

For the numerical evaluation, we set $d' = 4$ and we optimized the function on the domain $[-4, 5]^{d'}$. The results are provided in Figure~\ref{fig:powell}.

\begin{figure}[t]
    \centering
    \includegraphics[height=5.7cm]{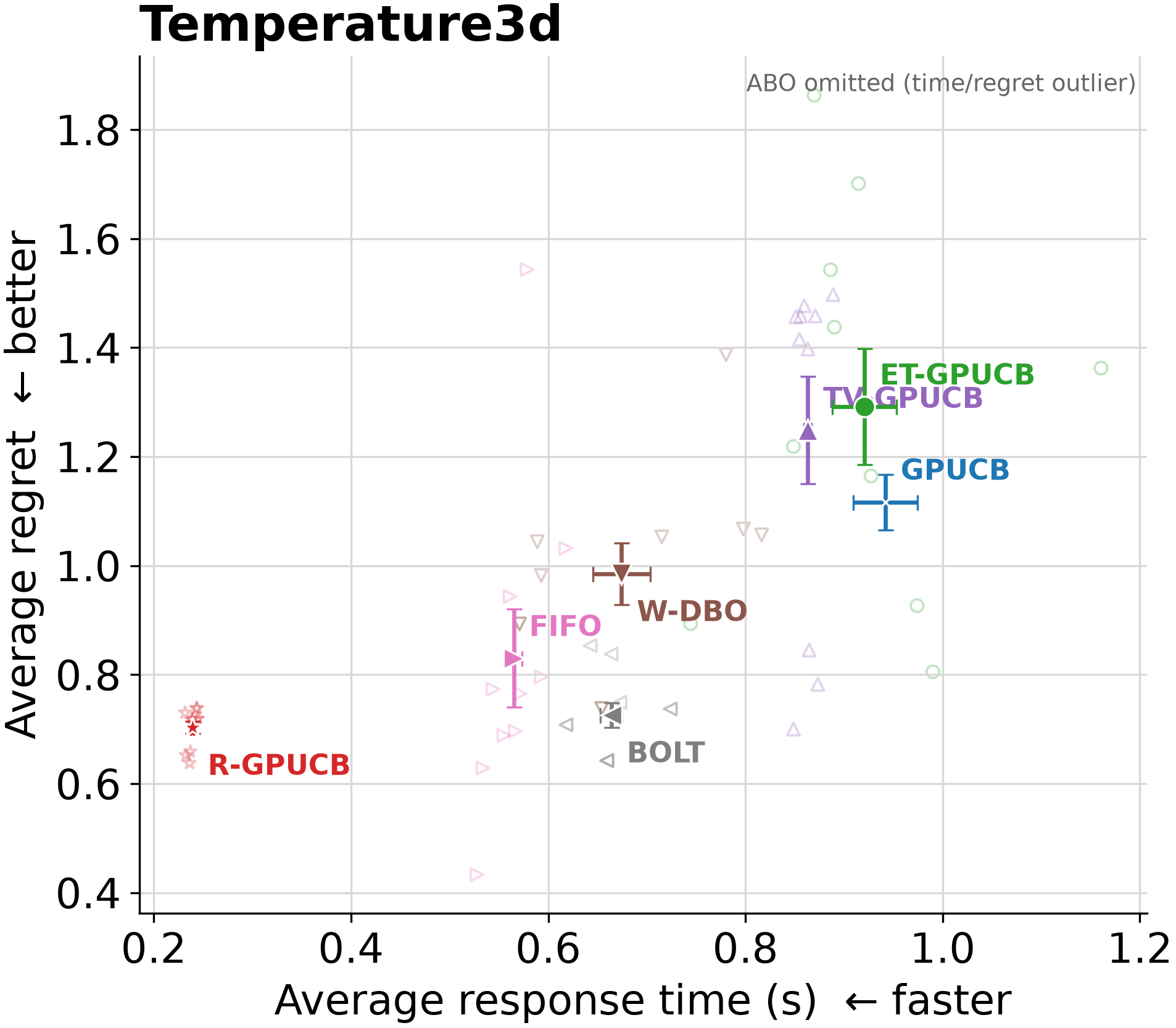} \quad
    \includegraphics[height=5.7cm]{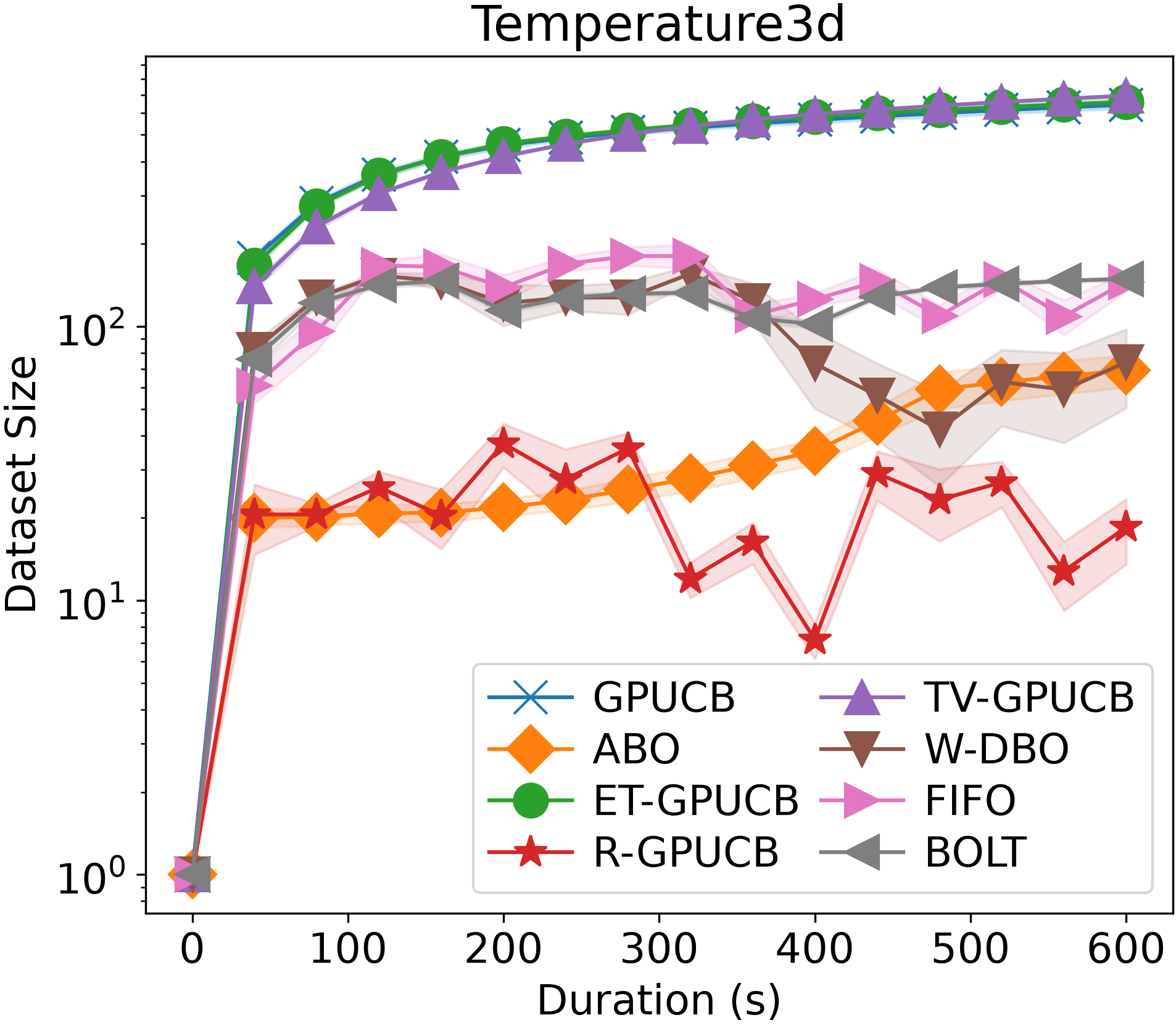}
    \caption{(Left) Average response time and average regrets of the TVBO solutions during the Temperature real-world experiment. (Right)~Dataset sizes of the TVBO solutions during the Temperature real-world experiment.}
    \label{fig:temperature}
\end{figure}

\paragraph{Temperature.} This benchmark comes from the temperature dataset collected from 46 sensors deployed at Intel Research Berkeley. It is a famous benchmark, used in other works such as~\citet{bogunovic2016time, brunzemaevent}. The goal of the TVBO task is to activate the sensor with the highest temperature, which will vary with time. To make the benchmark more interesting, we interpolate the data in space-time. With this interpolation, the algorithms can activate any point in space-time, making it a 3-dimensional benchmark (2 spatial dimensions for a location in Intel Research Berkeley, 1 temporal dimension).

For the numerical evaluation, we used the first day of data. The results are provided in Figure~\ref{fig:temperature}.

\begin{figure}[t]
    \centering
    \includegraphics[height=5.7cm]{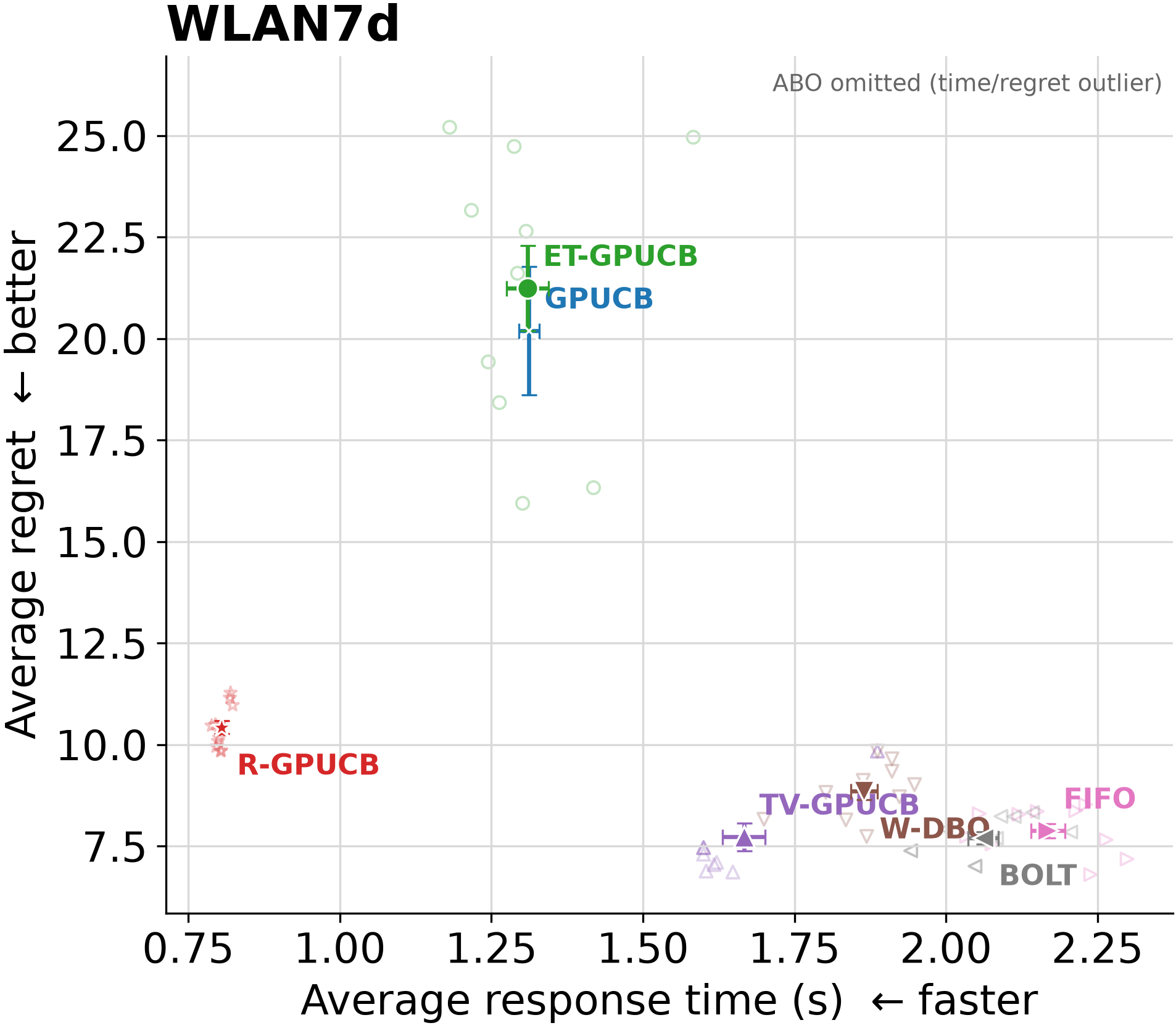} \quad
    \includegraphics[height=5.7cm]{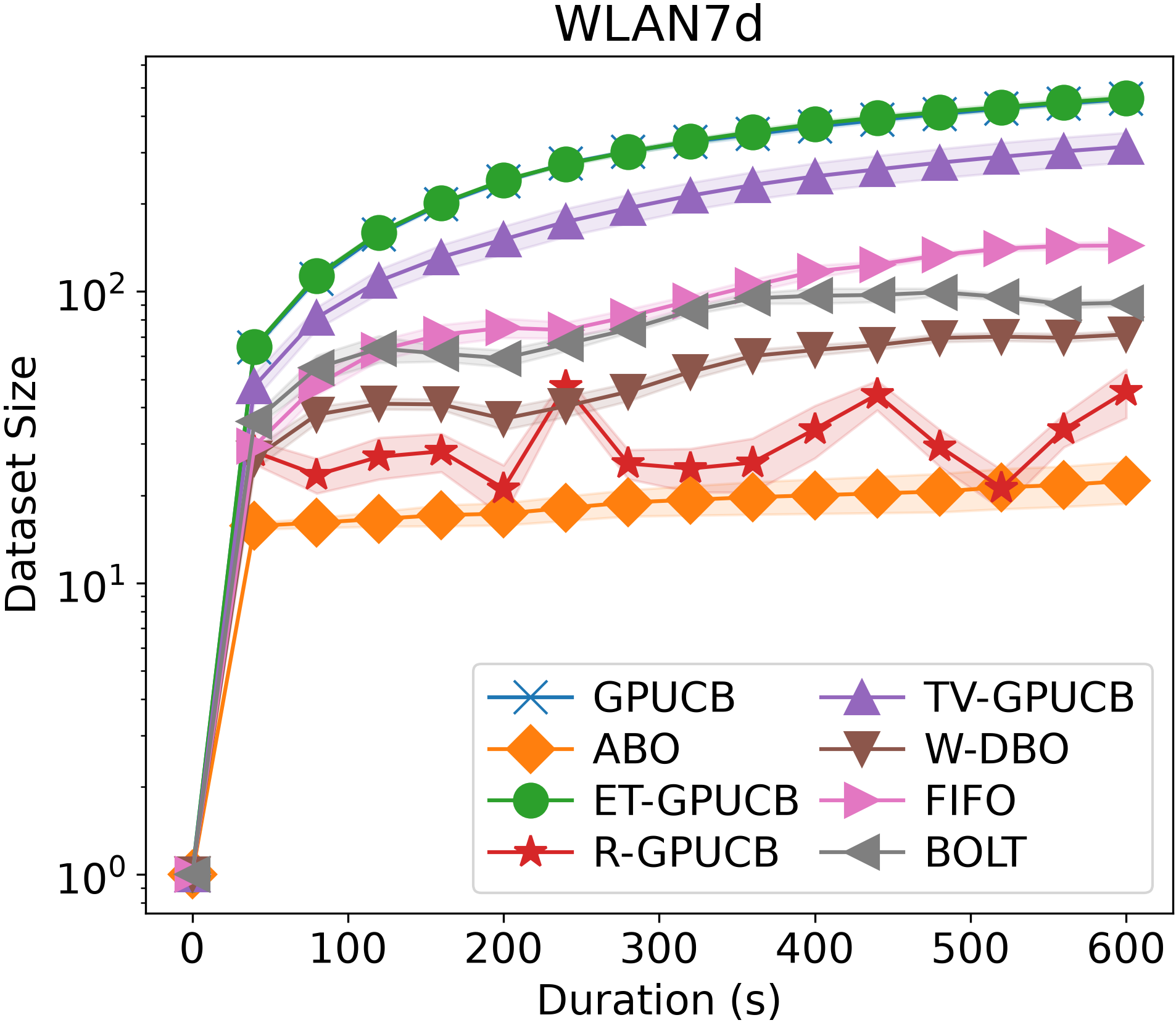}
    \caption{(Left) Average response time and average regrets of the TVBO solutions during the WLAN real-world experiment. (Right)~Dataset sizes of the TVBO solutions during the WLAN real-world experiment.}
    \label{fig:wlan}
\end{figure}

\paragraph{WLAN.} This benchmark aims at maximizing the throughput of a Wireless Local Area Network (WLAN). 27 moving end-users are associated with one of 6 fixed nodes and continuously stream a large amount of data. As they move in space, they change the radio environment of the network, which should adapt accordingly to improve its performance. To do so, each node has a power level that can be tuned for the purpose of reaching the best trade-off between serving all its users and not causing interference for the neighboring nodes.

The performance of the network is computed as the sum of the Shannon capacities for each pair of node and associated end-users. The Shannon capacity~\citep{kemperman-shannon:1974} sets a theoretical upper bound on the throughput of a wireless communication. We denote it $C(i, j)$, we express it in bits per second (bps). It depends on $S_{ij}$ the Signal-to-Interference plus Noise Ratio (SINR) of the communication between node $i$ and end-user $j$, as well as on $W$, the bandwidth of the radio channel (in Hz):
\begin{equation*}
    C_{ij}(\bm x, t) = W \log_2(1 + S_{ij}(\bm x, t)).
\end{equation*}

Then, the objective function is
\begin{equation*}
    f(\bm x, t) = \sum_{i = 1}^{6} \sum_{j \in \mathcal{N}_i} C_{ij}(\bm x, t),
\end{equation*}
where $\mathcal{N}_i$ are the end-users associated with node $i$.

For the numerical evaluation, we optimized the power levels $\bm x$ in the domain $[10^{0.1}, 10^{2.5}]^{6}$. For this experiment, the TVBO solutions were evaluated with a Matérn-5/2 for the spatial covariance function and a Matérn-1/2 for the temporal covariance function. The results are provided in Figure~\ref{fig:wlan}.

\end{document}